\documentclass[lettersize,journal]{IEEEtran}

\ifCLASSOPTIONcompsoc
  \usepackage[nocompress]{cite}
\else
  \usepackage{cite}
\fi
\usepackage{amsmath,amsfonts}
\usepackage{algorithmic}
\usepackage{algorithm}
\usepackage{array}
\usepackage{textcomp}
\usepackage{stfloats}
\usepackage{hyperref}
\usepackage{url}
\usepackage{verbatim}
\usepackage{graphicx}
\usepackage{cite}
\usepackage{booktabs}
\usepackage{multirow}
\usepackage{bbding}
\usepackage{subfigure}

\newcommand{\red}[1]{{\color{red}{#1}}}

\newcommand{\green}[1]{{\color{green}{#1}}}
\usepackage{color}
\usepackage{xcolor}
\usepackage{amsthm}
\definecolor{gre}{RGB}{144,186,179}
\newtheorem{prop}{Proposition}
\begin{document}

\title{SIEFormer: Spectral-Interpretable and -Enhanced Transformer for Generalized Category Discovery}

\author{Chunming Li, Shidong Wang, Tong Xin, and Haofeng Zhang 
\thanks{This work was partly supported by the National Natural Science Foundation of China (NSFC) under Grant Nos. 62371235 and U25A20444,  partly by the Key Research and Development Plan of Jiangsu Province under Grant BE2023008-2. (Corresponding author: Haofeng Zhang.)}
\thanks{Chunming Li and Haofeng Zhang are with the School of Computer Science and Engineering, Nanjing University of Science and Technology, Nanjing, 210094, China. E-mail:\{chunmingli, zhanghf\}@njust.edu.cn}
\thanks{Shidong Wang is with the School of Engineering, Newcastle University, Newcastle upon Tyne, NE17RU, United Kingdom. (e-mail: shidong.wang@newcastle.ac.uk)}
\thanks{Tong Xin is with the School of Computing, Newcastle University, Newcastle upon Tyne, NE17RU, United Kingdom. (e-mail: tong.xin@newcastle.ac.uk)}
}


\maketitle

\begin{abstract}
This paper presents a novel approach, Spectral-Interpretable and -Enhanced Transformer (SIEFormer), which leverages spectral analysis to reinterpret the attention mechanism within Vision Transformer (ViT) and enhance feature adaptability, with particular emphasis on challenging Generalized Category Discovery (GCD) tasks. The proposed SIEFormer is composed of two main branches, each corresponding to an \textit{implicit} and \textit{explicit} spectral perspective of the ViT, enabling joint optimization. The implicit branch realizes the use of different types of graph Laplacians to model the local structure correlations of tokens, along with a novel Band-adaptive Filter (BaF) layer that can flexibly perform both band-pass and band-reject filtering. The explicit branch, on the other hand, introduces a Maneuverable Filtering Layer (MFL) that learns global dependencies among tokens by applying the Fourier transform to the input ``value" features, modulating the transformed signal with a set of learnable parameters in the frequency domain, and then performing an inverse Fourier transform to obtain the enhanced features. Extensive experiments reveal state-of-the-art performance on multiple image recognition datasets, reaffirming the superiority of our approach through ablation studies and visualizations. Code is available at: \textcolor{red}{\url{https://github.com/Ashengl/SIEFormer}}.
\end{abstract}


\section{Introduction}
\IEEEPARstart{V}{isual} recognition, a cornerstone of computer vision, has witnessed remarkable advancements in the past decade, fueled by the dynamic evolution of deep learning methods \cite{resnet, MAE, ViT} and the availability of publicly accessible large-scale datasets \cite{imagenet}. While the use of deep learning techniques has pushed visual recognition systems to unprecedented heights, a key limitation remains - these systems heavily rely on a fundamental assumption that the data used for training, validation and test, adhere to the identical label distribution space or domain. Transitioning deep learning from controlled environments to open-world settings is a formidable and challenging task because it requires innovative approaches to empower visual recognition systems in dynamic data landscapes.  

\begin{figure}[t]
	\centering
	\includegraphics[width=0.99\columnwidth]{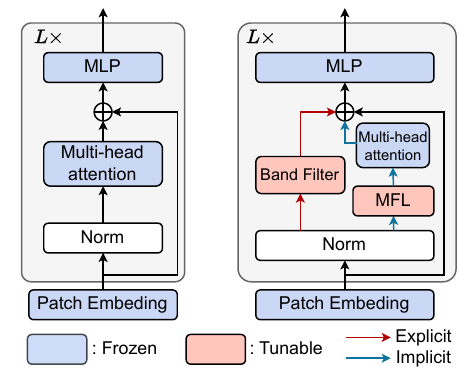} 
 \vspace{-5ex}
	\caption{Comparison between the standard Vision Transformer (ViT) on the left and the proposed SIEFormer on the right. SIEFormer introduces two spectral views: \textit{implicit}, using graph Laplacian filtering (Band-adaptive Filter, BaF) on \textit{value} features generated from the self-attention, and \textit{explicit}, applying the Fourier transform (Maneuverable Filtering Layer, MFL) for joint optimization and the above optimization strategy is beneficial to generate discriminative features which is helpful to discover novel categories.}
 \vspace{-3ex}
	\label{fig1}
\end{figure}

To realize this goal, researchers have explored various research areas, including Semi-Supervised Learning \cite{SSL0}, which leverages both labeled and unlabeled data for robust model training, Zero-Shot Learning \cite{lampert2013attribute} for classifying objects from previously unseen classes without specific training, Open-Set Recognition \cite{scheirer2012toward} that identifies whether unlabeled images belong to known classes, and Novel Category Discovery (NCD) \cite{Han2019learning}. Initially, NCD assumed that all unannotated images belong to unknown categories, which is less practical. Generalized Categories Discovery (GCD) \cite{GCD} further relaxed the assumption of NCD by introducing a more realistic setting where unlabeled images could belong to both known and unknown categories.

Through examining the existing methods, the accomplishments in GCD tasks can be roughly ascribed to two primary aspects. Early approaches, on the one hand, focus on extracting generic features using contrastive learning and then employing clustering algorithms to discover new categories. For instance, GCD \cite{GCD} employed both supervised contrastive learning and self-supervised contrastive learning, followed by a semi-supervised k-means algorithm to discover new categories. The recently proposed approaches, on the other hand, aim to use parametric classifiers in other unsupervised or self-supervised learning manner to further improve the ability to discover new categories. For example, SimGCD \cite{wen2023simgcd} introduced a parametric classifier and employed mean-entropy-maximization regularizer as well as knowledge distillation to guide the model to discover new categories. Most existing methods rely on a DINO \cite{DINO} pre-trained Vision Transformer (ViT) \cite{ViT} as a feature extractor, typically fine-tuning only the final attention block to maintain generalization. This reliance on pre-trained models facilitates the extraction of generic features and helps prevent overfitting to either known or unknown categories in GCD tasks \cite{promptcal}. Although fine-tuning additional layers may partially address the misalignment between the training objectives of pre-trained models and the specific goals of GCD, there remains a lack of architectures explicitly designed to meet the unique challenges of GCD.

An ideal GCD model should not only leverage the knowledge from pre-training but also adapt its architecture to effectively distinguish both seen and unseen samples. In this paper, we introduce a novel architecture, named Spectral-Interpretable and -Enhanced Transformer (SIEFormer), grounded in spectral theory—a model that capitalizes on spectral perspectives to enhance model interpretation and feature adaptability. There are also several works leveraging spectral analysis to enhance Vision Transformer. Reza \textit{et al}. \cite{azad2023laplacian} proposed constructing a Laplacian pyramid to fuse the feature for ViT at each level. Devin \textit{et al}. \cite{kreuzer2021rethinking} proposed to take advantage of the full Laplacian spectrum to learn the position of each node in a given graph. Li \textit{et al}. \cite{kreuzer2021rethinking} proposed to split spectral information of hyperspectral images to enhance ViT. However, SIEFormer considers utilizing spectral analysis to enhance attention mechanism, which is suitable for fine-tuning with contrastive learning. Concretely, the affinity kernels derived from the product of the \textit{query} and \textit{key} features within the Vision Transformer (ViT) can be transformed and reinterpreted as symmetric adjacency matrices through the spectral analysis, which is a technique to study the graph's spectral properties, particularly those related to the graph Laplacian containing information about the graph structure, allowing the introduction of different forms of graph Laplacian, ranging from the linear (e.g., All-Pass Filter) to the polynomial (e.g., Chebyshev Filter \cite{zhu2021unifying}) and even the rational (e.g., ARMA Filter \cite{bianchi2021graph}, Cayley Filter \cite{cayley}), to be adopted for \textbf{implicitly} re-weighting the \textit{value} features. We draw inspiration from the graph Laplacian to design a novel filter layer that adaptively fuses the characteristics of both band-pass and band-reject filters in the frequency domain. This approach effectively suppresses irrelevant noisy components while preserving discriminative ones, producing feature distributions that are semantically more compact and separable. As a result, the resulting feature space is more structured and interpretable, which in turn facilitates subsequent contrastive learning and improves the ability to discriminate between categories.

In parallel with the implicit branch, we also present a Maneuverable Filtering Layer (MFL) that incorporates a set of learnable parameters for filtering out noises within the \textit{value} features by \textbf{explicitly} using the Fourier transform and its inverse. The above filtering operation benefits contrastive learning by using the filter design to adaptively remove high-frequency signals (i.e., noisy features) or low-frequency signals (i.e., generic features), thereby uncovering class-specific features that enhance the ability to discover new categories.
In summary, our contributions can be outlined as follows:
\begin{itemize}
    \item We introduce a Spectral-Interpretable and -Enhanced Transformer (SIEFormer), which reinterprets attention mechanism within the ViT by using spectral analysis, enhancing feature interpretability and adaptability from two spectral views and challenging GCD tasks.
    \item An \textbf{implicit} branch is present to realize the use of different forms of the graph Laplacian to re-weight \textit{value} features, along with a novel Band-adaptive Filter (BaF) layer to model the local neighborhoods between tokens while maintaining adaptability between band-pass and band-reject filtering.
    \item An \textbf{explicit} branch, the core of which is a Maneuverable Filtering Layer (MFL), is also proposed to learn the global dependencies among tokens by modulating the Fourier-transformed \textit{value} features with a set of learnable parameters in the frequency domain and then reconstructing the filtered features through an inverse Fourier transform.
\end{itemize}

Existing methods typically enhance ViTs through frequency domain filtering via direct filtering \cite{zeng2023bandpass, azad2023laplacian}, spectral analysis \cite{wang2022antioversmooth,bai2022improving,wang2022vtc}, or graph Laplacian operations \cite{tse2023spectral,karmim2024supra}, often treating self-attention as a low-pass filter. In contrast, SIEFormer models self-attention as a graph and leverages spectral analysis to interpret and enhance the attention mechanism. More importantly, the integration of the dual-spectral views, as illustrated in Fig. \ref{fig1}, allows for joint optimization and generates discriminative feature to discover novel categories. Extensive experiments were conducted, leading to remarkable results that set new state-of-the-art performance on generic image recognition datasets and fine-grained image datasets. Ablation studies and visualizations further underscore the superiority of the proposed approach.

\section{Related Work}\label{sec2}
\subsection{Semi-Supervised Learning}
Semi-supervised learning (SSL) has been proposed and widely explored \cite{SSLsurvey, SSL0, semi_TNNLS1} to solve the problem of insufficient labeled sample volume. The problem that SSL addresses is achieving high classification accuracy with a small number of labeled samples and a large number of unlabeled samples coming from the same sample space. One of the most basic SSL approaches is the self-training behind which the fundamental idea is to augment the initial labeled data with the most confident predictions made on unlabeled data. Despite its simplicity, it has demonstrated compelling results, e.g., Pseudo-label \cite{PL} and MPL \cite{MPL}. Most recently, consistency regularization methods have risen to prominence. These methods encourage the model to make similar predictions under small perturbations, commonly applied to the input data. e.g., MixMatch \cite{mixmatch} proposed a fusion of consistency regularization and entropy minimization method; FixMatch \cite{fixmatch} introduced a fusion of consistency regularization with pseudo-labeling. Meanteacher \cite{meanteacher} updates the model by employing consistency regularization on student and teacher models' parameters. \textit{SSL relies on a small amount of labeled data and typically assumes that the unlabeled data does not contain unseen classes. Recently, Category Discovery relaxes this assumption by leveraging structural information from the labeled data to discover and distinguish novel classes within the unlabeled data  automatically.}

\subsection{Novel Category Discovery}
Unlike SSL, the goal of NCD is to discover new categories in unlabeled data using labeled data \cite{restune}. Originally normalized by DTC \cite{Han2019learning}, which has been proposed based on the unsupervised clustering method DEC \cite{DEC}, this approach first uses the ground truth of the labeled data to train the feature extractor using cross-entropy loss, and then categorizes the instances by maintaining a list of prototypes. Autonovel \cite{autonovel1} proposed to use self-supervised pre-training, and employ ranking to generate pseudo-labels for samples to train the model jointly with the ground truth of the labeled data. OpenMix \cite{OpenMix} utilizes the MixUp \cite{mixup} to generate new samples with labeled and unlabeled samples, together with the generation of more robust pseudo-labels. Inspired by Autonovel \cite{autonovel1, autonovel2}, NCL \cite{NCL} introduces contrastive learning to improve the ability of the model to generate discriminative representations. UNO \cite{UNO} proposes a way to jointly train models using labeled data as well as unlabeled data. Recently, NCD has also been deployed to scenarios, including the semantic segmentation \cite{ncdss}, which distinguishes objects and backgrounds while discovering novel categories, and the incremental learning-based NCD that proposes to discover new classes while maintaining the ability to recognize previously seen classes \cite{incd}.

The GCD problem was first introduced by \cite{GCD} and has gained widespread attention. The GCD task proposes to use DINO pre-trained ViT-B/16 \cite{DINO}, and utilize self-supervised learning with contrastive learning for fine-tuning and semi-supervised clustering in the testing phase. Similar to the GCD task setup, Open Set Domain Adaptation \cite{OSDA} and Open-World Semi-Supervised Learning \cite{openldn} are both designed to address the problem of new class discovery in an open-world setup. DCCL \cite{DCCL} proposed to alternately estimate underlying visual conceptions and learn conceptional representation. SimGCD \cite{wen2023simgcd} proposed a parametric classification method that benefits from entropy regularization, and achieves state-of-the-art performance on multiple GCD benchmarks. However, the GCD algorithm encounters some limitations. Concretely, freezing the backbone network prevents it from quickly adapting to downstream tasks, while a large number of false negatives can lead to degradation of the semantic representation \cite{promptcal}. To address these problems, PromptCAL \cite{promptcal} proposed prompt-based contrastive affinity learning, while SPTNet \cite{wang2024sptnet} introduced the spatial prompt tuning approach. \textit{Despite these advancements, the inability of the backbone network to better adapt to downstream tasks is still not well addressed. Therefore, this research proposes the SIEFormer, which capitalizes on spectral perspectives to enhance feature understanding and adaptability.}

\subsection{Graph Filtering} 
At the intersection between deep learning and methods for modeling structured data, the function of Graph Neural Networks (GNNs) is to infer arbitrary relationships present in the discrete data by extracting reliable representations \cite{graph_TNNLS3}. Existing GNNs can be primarily categorized into two groups. Spatial-based GNNs \cite{velickovic2017graph, hamilton2017inductive} focus on the propagation and aggregation of messages between neighboring nodes, while spectral-based GNNs \cite{kipf2016semi, defferrard2016convolutional, klicpera2018predict, zhu2021interpreting, graph_TNNLS2} generally take full advantage of the graph Laplacian created in the spectral domain. According to whether the graph filter can be dynamically updated, spectral-based GNNs can be further divided into two categories: methods based on predetermined filters and those based on learnable filters. One of the most well-known predetermined-based graph convolutions is GCN \cite{kipf2016semi}, where graph filters are approximated as simplified $1^{st}$-order Chebyshev polynomials. The Personalized Propagation of Neural Predictions (PPNP) \cite{klicpera2018predict} and its fast approximation, APPNP, are derived from the Personalized PageRank and realize the ability to have low-pass filters. More recently, Zhu et al. \cite{zhu2021interpreting} derived adjustable graph kernels that can display low-pass or high-pass filtering capabilities on top of a unified target optimization framework with a feature-fitting function and a graph regularization term.

Polynomial filters have the advantage of having a finite impulse response to the signal and performing weighted moving average filtering on their local neighborhoods, enabling the use of learnable Chebyshev polynomials, such as ChebNet \cite{tang2019chebnet}. Nevertheless, polynomial filters are difficult to model sharp changes in signal frequency response due to their smoothness \cite{isufi2016autoregressive, tremblay2018design}, leading to a growing interest in exploring rational filters. For example, Levie et al. \cite{levie2018cayleynets} designed a CayleyNet that approximates the effect of rational filters, but it shares high similarities with the Chebyshev filters in that it uses a fixed number of Jacobi iterations to approach the matrix inversion with a sequence of differentiable operations. Then, Bianchi et al. \cite{bianchi2021graph} proposed to use a more advanced rational filtering approach, namely Auto-Regressive Moving Average (ARMA) filters that integrate polynomial filters, also treated as the MA term, with an additional AR term \cite{zhou2003learning, zhu2003semi}. \textit{Inspired by graph filtering, this paper proposes that the attention mechanism in ViTs can be interpreted as a graph, and considers leveraging spectral analysis to enhance the attention mechanism. Our analysis further demonstrates that such a spectral perspective is more suitable for GCD.}

\begin{figure*}[!t]
\centering
\includegraphics[width=0.96\textwidth]{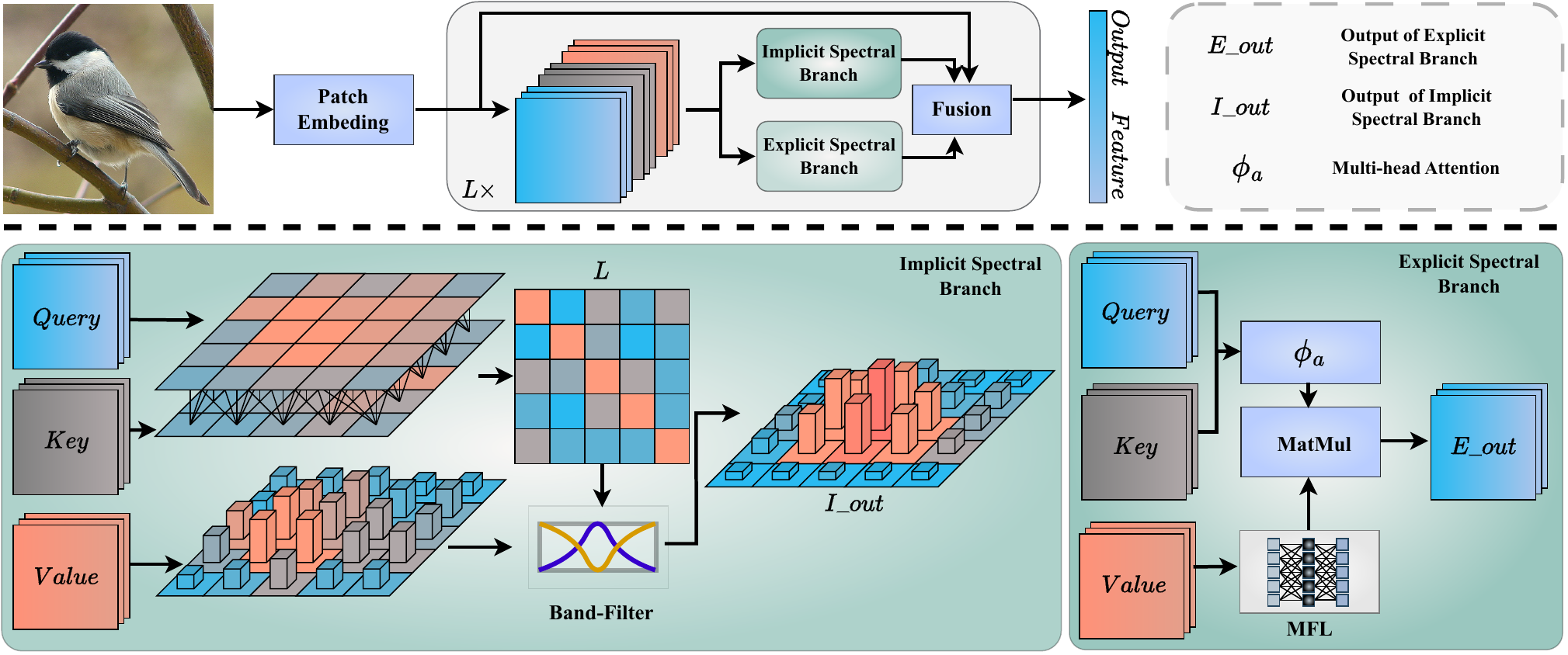}%
\vspace{-2ex}
\caption{Overview of SIEFormer. Implicit Spectral Branch represents the implicit use of Band-adaptive Filter to operate the \textit{values} in self-attention with the eigenvalues of the Laplace matrix, and Explicit Spectral Branch stands for directly converting \textit{values} to the spectral domain using the fast Fourier transform and reconstructing \textit{values} using MFL. }\label{fig_main}
\vspace{-3ex}
\end{figure*}

\section{Methodology}\label{sec3}
\begin{table}[!t]
\centering
\caption{Differences between Semi-Supervised Learning (SSL), Novel Category Discovery (NCD) and Generalized Category Discovery (GCD).}
\label{diff}
\vspace{-2ex}
\begin{tabular}{lcc}
\toprule
\multicolumn{1}{l}{Task} & \multicolumn{1}{c}{Training Dataset} & \multicolumn{1}{c}{Relationship: $\mathcal{Y}_l$ and $\mathcal{Y}_u$ }\\ \cmidrule(l){1-3} 

\multicolumn{1}{l|}{SSL} & \multicolumn{1}{c}{$\mathcal{D}=\mathcal{D}_l \cup \mathcal{D}_u$} & \multicolumn{1}{|c}{$\mathcal{Y}_l=\mathcal{Y}_u$}\\ \cmidrule(l){1-1}\cmidrule(l){3-3}

\multicolumn{1}{l|}{NCD} & \multicolumn{1}{c}{$\mathcal{D}_u=\{(x_i, y_i)\} \in \mathcal{X} \times \mathcal{Y}_u $} & \multicolumn{1}{|c}{$\mathcal{Y}_l\cap \mathcal{Y}_u=\phi$}\\ \cmidrule(l){1-1}\cmidrule(l){3-3}

\multicolumn{1}{l|}{GCD} & \multicolumn{1}{c}{$\mathcal{D}_l=\{(x_i, y_i)\} \in \mathcal{X} \times \mathcal{Y}_l $} & \multicolumn{1}{|c}{$\mathcal{Y}_l\in \mathcal{Y}_u$}\\
\bottomrule
\vspace{-5ex}
\end{tabular}
\end{table}

\subsection{Preliminaries}
\subsubsection{Problem Definition}
In the context of GCD tasks, it is assumed that the training dataset $\mathcal{D}$ is composed of a labeled sample set $\mathcal{D}_l$ and an unlabeled sample set $\mathcal{D}_u$, such that $\mathcal{D} = \mathcal{D}_l \cup \mathcal{D}_u$. The goal of GCD \cite{GCD} is to learn representations on $\mathcal{D}_l$ using clustering algorithms with a prior on the number of unlabeled classes, which can then be used to identify test samples. The differences between SSL, NCD, and GCD tasks are summarized and shown in Table \ref{diff}. In SSL, the objective is to leverage both labeled and unlabeled data to improve the performance on the known classes. In contrast, the goal of NCD is to discover new categories in unlabeled data using labeled data. The GCD task combines elements of both SSL and NCD, aiming to learn representations on the labeled data and identify test samples belonging to both known and unknown classes.

\subsubsection{Self-Attention in Transformer}
The core idea of the Transformer \cite{vaswani2017attention} is the use of the self-attention mechanism to re-weight the input features in an embedding space, followed by scaling the re-weighted features to the original shape of the input features. This concept employs a recursive-free aggregation scheme to effectively capture global dependencies in feature representations, which has inspired the development of Vision Transformer (ViT) \cite{ViT} to the computer vision community. In ViT, the input features can be denoted as $\mathbf{X} \in \mathbb{R}^{N\times C}$, where $N$ stands for the number of patches generated using the patch and position embedding approach, and each patch is encoded as a $C$-dimensional feature vector. The linear mapping functions $\mathrm{Linear}_Q(\cdot)$, $\mathrm{Linear}_K(\cdot)$ and $\mathrm{Linear}_V(\cdot)$ are adopted to project $\mathbf{X}$ into three $C_e$-dimensional spaces, corresponding to the matrices of $queries : \mathbf{Q}$, $ keys : \mathbf{K}$ and $values : \mathbf{V}$. Formally, this process is denoted as
\begin{equation}
\label{QKV}
\begin{aligned}
\mathbf{Q}&=\mathrm{Linear}_Q(\mathbf{X})=\mathbf{XW}_Q \in \mathbb{R}^{N\times C_e},\\
\mathbf{K}&=\mathrm{Linear}_K(\mathbf{X})=\mathbf{XW}_K \in \mathbb{R}^{N\times C_e},\\
\mathbf{V}&=\mathrm{Linear}_V(\mathbf{X})=\mathbf{XW}_V \in \mathbb{R}^{N\times C_e},
\end{aligned}
\end{equation}
where $\mathbf{W}_Q$,$\mathbf{W}_K$,$\mathbf{W}_V$ are learnable weights. 

The self-attention operation can be expressed in a generic formulation as
\begin{equation}
\label{ZQKV}
\mathbf{Z}_{i,:}=\frac{1}{N}\sum_{\forall j}\omega(\mathbf{Q}_{i,:}, \mathbf{K}_{j,:})\odot \mathbf{V}_{j,:},
\end{equation}
where $\odot$ is the Hadamard product, and $i$,$j \in \{1,\ldots,n\}$ are the position of the projected feature maps. The self-attention function $\omega (\cdot) : \mathbb{R}^{C_e}\times \mathbb{R}^{C_e} \to \mathbb{R}$ consists of a nonlinear function $\mathrm{softmax(\cdot)}$ and a relation function $\tau(\cdot): \mathbb{R}^{C_e}\times \mathbb{R}^{C_e} \to \mathbb{R}$ that is typically instantiated as
\begin{equation}
\label{TAU}
\tau(\mathbf{Q}_{i,:}, \mathbf{K}_{j,:})=\frac{\mathbf{Q}_{i,:}\mathbf{K}_{j,:}^T}{\sqrt{C_e}},
\end{equation}
where $\tau(\mathbf{Q}_{i,:}, \mathbf{K}_{j,:})$ can be regarded as the affinity kernel calculated by pairwise summing the feature at the position $i$ and those at all possible position $j$ \cite{vaswani2017attention, ViT}.  
The re-weighed feature $\mathbf{Z}_{i,:}$ in Eq. \eqref{TAU} is linearly transformed with a weight matrix $\mathbf{W}_Z \in \mathbb{R}^{C_e\times C}$ in order to revert to the original input dimension, alongside a standard residual connection to $\mathbf{X}$, resulting in the following representation
\begin{equation}
\label{YXZ}
\mathbf{Y}_{i,:}=\mathbf{X}_{i,:}+\mathbf{Z}_{i,:}\mathbf{W}_Z.
\end{equation}

\subsection{Spectral View of ViT}

Spectral analysis of the matrix, especially the spectral theorem\footnote{\url{https://en.wikipedia.org/wiki/Spectral_theorem}}, provides essential criteria for the existence of a canonical decomposition of the matrix, stating that the matrix involved should be diagonalizable, i.e., a matrix is equal to its transpose. However, in reality, the affinity matrix that encodes the token-wise relationships used in the original ViT model cannot guarantee its symmetry because the varying initializations of the weight matrices, $\mathbf{W}_Q$ and $\mathbf{W}_K$, are adopted to enhance the generalization ability of the model. Inspired by \cite{tao2018nonlocal, zhu2021unifying}, the affinity kernel summarising the token similarities of embedded matrices $\mathbf{W}_Q$ and $\mathbf{W}_K$, denoted $\mathbf{M}$, can be obtained by using $\mathbf{M}=[M_{ij}]=[\tau(\mathbf{Q}_{i,:}, \mathbf{K}_{j,:})]$ where $\tau(\cdot)$ represents the relation function of $\tau=\frac{\mathbf{Q}_{i,:}\mathbf{K}_{j,:}^T}{\sqrt{C_e}}$. The symmetric $\mathbf{\bar{M}}$ form of $\mathbf{M}$ can then be formed by
\begin{equation}
\label{BarM}
\mathbf{\bar{M}}=\frac{\mathbf{M}+\mathbf{M}^T}{2}.
\end{equation}

The transformed $\mathbf{\bar{M}}\in \mathbb{R}^{N\times N}$ is a symmetric matrix with $N$ real eigenvalue, so it is diagonalizable and has the potential to serve as an ensemble of graph filters (Implicit Spectral Branch in Fig. \ref{fig_main}). For clarity, some fundamental definitions and notations from spectral graph theory will be presented below, with a particular highlight on how the yielded $\mathbf{\bar{M}}$ can be exploited to neatly and plausibly interpret the ViT model from a spectral perspective.\\

\noindent\textbf{Graph Fourier Transform}
Given an undirected graph $\mathcal{G}=(\mathcal{V},\mathcal{E})$ consisting of a finite set of node $\mathcal{V}$ and an edge set $\mathcal{E}$, a graph signal or function $f:\mathcal{V}\to \mathbb{R}$ defined on the graph nodes, expressed as a vector $\boldsymbol{f}\in \mathbb{R}^N$, maps every node $\{v_i\}_{i=1,\ldots,N}$ in the set of nodes $\mathcal{V}$ with $\mathcal{V}=N$ to a real number $f(i)$ \cite{shuman2013emerging}. The adjacency matrix $\mathbf{A}\in\mathbb{R}^{N\times N}$ is used to characterize a graph with $N$ nodes and its symmetrically normalized form is denoted as $\Tilde{\mathbf{A}}=\mathbf{D}^{-\frac{1}{2}}\mathbf{A}\mathbf{D}^{-\frac{1}{2}}$ where $\mathbf{D}$ is degree matrix. The Laplacian matrix of graph is denoted by $\mathbf{L}=\mathbf{D}-\mathbf{A}\in \mathbb{R}^{N\times N}$ and popular symmetrically normalized graph Laplacian is defined as $\Tilde{\mathbf{L}}=\mathbf{D}^{-\frac{1}{2}}\mathbf{L}\mathbf{D}^{-\frac{1}{2}}=\mathbf{I}_N-\Tilde{\mathbf{A}}$, where $\mathbf{I}_N$ is the identity matrix with N identities. The graph Fourier transform $\boldsymbol{\hat{f}}$ of any graph signal $\boldsymbol{f}$ on the nodes of $\mathcal{G}$ can be defined as the expansion of $\boldsymbol{f}$ in terms of the eigenvectors derived from the eigenfunctions of the graph Laplacian $l$ and is denoted as
\begin{equation}
\label{FFT}
\mathcal{GF}\vert f\vert(\lambda_l)=\hat{f}(\lambda_l)=\left \langle \boldsymbol{f}, \boldsymbol{u}_l \right \rangle = \sum_{i=1}^Nf(i)u^*_l(i),
\end{equation}
where $u_l^*$ is equivalent to $u^T_l$ representing the $l_{th}$ eigenvector of the $\mathbf{L}$, and $\lambda_l$ is the corresponding $l_{th}$ eigenvalue. In view of the fact that the eigenvectors $\{u_l\}_{l=0,\ldots,N-1}$ of Laplacian matrix $\mathbf{L}$ constitute an orthogonal basis, the corresponding inverse graph Fourier transform exists, which can be expressed as $\mathcal{LGF}\vert\hat{f}\vert(i)=f(i)=\sum_{l=0}^{N-1} \hat{f}(\lambda_l)u_l(i).$\\

\noindent\textbf{Graph Filtering}
Since general filtering operators cannot be explicitly expressed in the node domain, a graph filter is introduced as an operator defined on the basis of the graph Fourier transform in order to be capable of passing or amplifying certain frequencies of the input signal while attenuating the other frequencies. Concretely, a frequency response function $g_{\theta}(\cdot)$ is introduced to act on each eigenvalue $\lambda$ of the eigenfunction and modify the components of the signal $\mathbf{V}$ by its product with the eigenvectors of $\Tilde{\mathbf{L}}$ \cite{defferrard2016convolutional, bianchi2021graph,zhu2021unifying}. Formally, this filtering process is denoted as
\begin{equation}
\label{FP}
\begin{aligned}
\hat{\mathbf{V}}&=\sum^N_{n=1}g_{\theta}(\lambda_n)u_nu_n^T\mathbf{V}\\
&=\mathbf{U}g_{\theta}(\mathbf{\Lambda})\mathbf{U}^T\mathbf{V},
\end{aligned}
\end{equation}
where $\mathbf{\Lambda}=diag[\lambda_0,\ldots,\lambda_{N-1}]\in \mathbb{R}^{N\times N}$, $\mathbf{U}=[u_0,\ldots,u_{N-1}]\in \mathbb{R}^{N\times N}$, and $g_{\theta}(\mathbf{\Lambda})=diag(\theta)$ denotes the non-parametric filter, with a vector of Fourier coefficients denoted as $\theta\in\mathbb{R}^R$ \cite{defferrard2016convolutional}.

\noindent\textbf{Spectral ViT}
By emulating the previously defined graph filtering, the parameter matrix $\mathbf{W}_Z\in \mathbb{R}^{C_e\times C}$ originated from the ViT model can be decomposed into two matrices $\mathbf{W}_{Z1}\in \mathbb{R}^{N\times N}$, where $\mathbf{W}_{z1}$ can be viewed as a generalization form of a set of graph filters. By substituting the two decomposed matrices into Eq. \eqref{YXZ}, it results in the following rewritten form
\begin{equation}
\label{DEC}
\mathbf{Y}=\mathbf{X}+\mathbf{TVW}_{Z1}\mathbf{W}_{Z2},
\end{equation}
where $\mathbf{T}=\mathbf{D}^{-1}_{\bar{\mathbf{M}}}\bar{\mathbf{M}}\in \mathbb{R}^{N\times N}$ is updated symmetric affinity kernel, with the diagonal matrix $\mathbf{D}_{\bar{\mathbf{M}}}\in \mathbb{R}^{N\times N}$ represents the degrees of all vertices of $\bar{\mathbf{M}}$ acquiring by using $\mathbf{D}_{ii}=\sum_j\bar{\mathbf{M}}_{ij}$, and $\mathbf{V}\in \mathbb{R}^{N\times C_e}$ denotes the linearly transferred node features, refers to $values$ in ViT. This equation can be trivially restored to Eq. \eqref{YXZ} by calculating the product of $\mathbf{T}$ and $\mathbf{V}$, and the unfactorized parameter $\mathbf{W}_Z$. Note that the indices are omitted in Eq. \eqref{DEC} and sequel for readability. After some arrangement, the later term of Eq. \eqref{DEC} can be represented as an operator in the form of $\mathcal{S}(\mathbf{T}, \mathbf{V})$ which is conducive to theoretically viewing and interpreting the ViT model from a spectral perspective by operating the spectral filtering on a fully-connected graph with the node feature $\mathbf{V}$ and the affinity kernel $\mathbf{T}$ as inputs. Our derived operator, capable of capturing the correlations between tokens, operates at a higher level of abstraction compared to the operator proposed in \cite{zhu2021unifying}, which focuses solely on the pixel level.

With the redefined spectral Vision Transformer (ViT), it becomes crucial to select appropriate graph filters and apply them to the node features. The linear filter function $g$ controlled by the parameter $\theta$ in Eq. \eqref{FP} is the most intuitive way to adjust the weights of the frequency components, however, it suffers from several significant drawbacks. Specifically, the spectral filters introduced in Eq. \eqref{FP} are not localised because the complete projection on the eigenvectors implies that the filters take into account the interactions of an individual node with the entire graph which weakens the significance of local node neighbourhoods. Furthermore, the expensive computation costs are not only reflected in the eigen-decomposition function of $\Tilde{\mathbf{L}}$, but also involve the calculation of the double product with the eigenvector basis $\mathbf{U}$. In addition, since the filters are determined by a specific Laplacian spectrum, it is problematic to adaptively transfer to modelling graphs with different structures \cite{bianchi2021graph}.

To cope with the limitations of the linear frequency response function, a filter function composed of a set of polynomials considering higher-order frequencies is used to perform a weighted moving average over the input signal to improve its representation ability \cite{tremblay2018design} which can be defined as $g_\theta(\mathbf{\Lambda})=\sum_{k=0}^{K-1}\theta_k\mathbf{\Lambda}^k$, where $\theta\in \mathbb{R}^K$ is the vector of polynomial coefficients. In reality, it is recommended to set the order of these general-purpose polynomial filters small, so that the learning complexity is greatly reduced from $\mathcal{O}(N)$ to $\mathcal{O}(K)$, namely, from the data dimension $N$ to the polynomial filter order $K$, where $K\ll N$. Furthermore, spectral filters represented by polynomial filters are exactly linear combinations of nodes and corresponding $K$-hop neighbourhoods, meaning that these filters are localised in the node space \cite{defferrard2016convolutional}, where these polynomial filters can be expressed by taking the $k$-th power of the normalized Laplacian matrix $\Tilde{\mathbf{L}}$ as it is diagonalizable, namely, $\Tilde{\mathbf{L}}^k=\mathbf{U}diag[\lambda^k_0,\ldots,\lambda^k_{N-1}]\mathbf{U}^T$. Then, the filtering operation of graph signals becomes
\begin{equation}
\label{POL}
\hat{\mathbf{V}}=(\theta_0\mathbf{I}+\theta_1\mathbf{\Tilde{L}}+\theta_2\mathbf{\Tilde{L}}^2+\cdots+\theta_K\mathbf{\Tilde{L}}^K)\mathbf{V}.
\end{equation}

Chebyshev polynomials, one of the most well-known polynomial functions, are often employed to attenuate undesired oscillations around the cutoff frequency responses of $\Tilde{\mathbf{L}}$. Such polynomials can also be utilised to approximate graph filters as in \cite{wang2018non} by rewriting the recursive Chebyshev expansion originally proposed in \cite{defferrard2016convolutional} as $T_k(\mathbf{\Tilde{L}})=2\mathbf{\Tilde{L}}T_{k-1}(\mathbf{\Tilde{L}})-T_{k-2}(\mathbf{\Tilde{L}})$. Formally, a truncated expansion used to parameterise the filters is expressed as
\begin{equation}
\label{CHEB}
\hat{\mathbf{V}}=\sum_{k=0}^{K-1}\theta_kT_k(\Tilde{\mathbf{L}})\mathbf{V},
\end{equation}
where $T_0(\mathbf{\Tilde{L}})=\mathbf{I}_N$, $T_1(\mathbf{\Tilde{L}})=\mathbf{\Tilde{L}}$, and $\mathbf{\Tilde{L}}=2\mathbf{L}/\lambda_{max}-\mathbf{I}_N$ is a scales Laplacian matrix. The spectral filtering operation defined above is independent of the Laplacian spectrum \cite{zhang2018end} and its Chebyshev approximation of the filters makes it applicable for fast filtering of graph signals (a.k.a., ChebNet \cite{defferrard2016convolutional}). The use of Chebyshev filters in the context of the spectral graphs is generalized to formulate and explain the non-local neural networks \cite{wang2018non} from the spectral view by treating $\mathbf{V}$ described in $\eqref{DEC}$ as the input graph signal. More specifically, a generalized form of the spectral self-attention operator with Chebyshev filters can be denoted as
\begin{equation}
\label{cheb}
\mathcal{S}_{\mathrm{cheb}}(\mathbf{T}, \mathbf{V})=\mathbf{VW}^1_Z+\mathbf{TVW}^2_Z+\sum_{k=2}^{K-1}\mathbf{T}^K\mathbf{VW}_Z^{k+1},
\end{equation}
where $\mathbf{T}=\mathbf{D}_{\Tilde{\mathbf{M}}}^{-\frac{1}{2}}\Tilde{\mathbf{M}}\mathbf{D}_{\Tilde{\mathbf{M}}}^{-\frac{1}{2}}$ denotes the normalized matrix of $\Tilde{\mathbf{M}}$. 

Considering the fact that the most similar tokens to a given token are the token itself and its surrounding tokens, we can approximate this local neighborhood by taking the $1^{st}$-order approximation of the Chebyshev filter \cite{zhu2021unifying}. Consequently, the operation from the spectral viewpoint of self-attention can be formulated as follows:
\begin{equation}
\label{1cheb}
\mathcal{S}_{\mathrm{cheb}}(\mathbf{T}, \mathbf{V})=\mathbf{VW}^1_Z+\mathbf{T}\mathbf{VW}^2_Z.
\end{equation}

\begin{figure}[!t]
\centering
\includegraphics[width=0.49\textwidth]{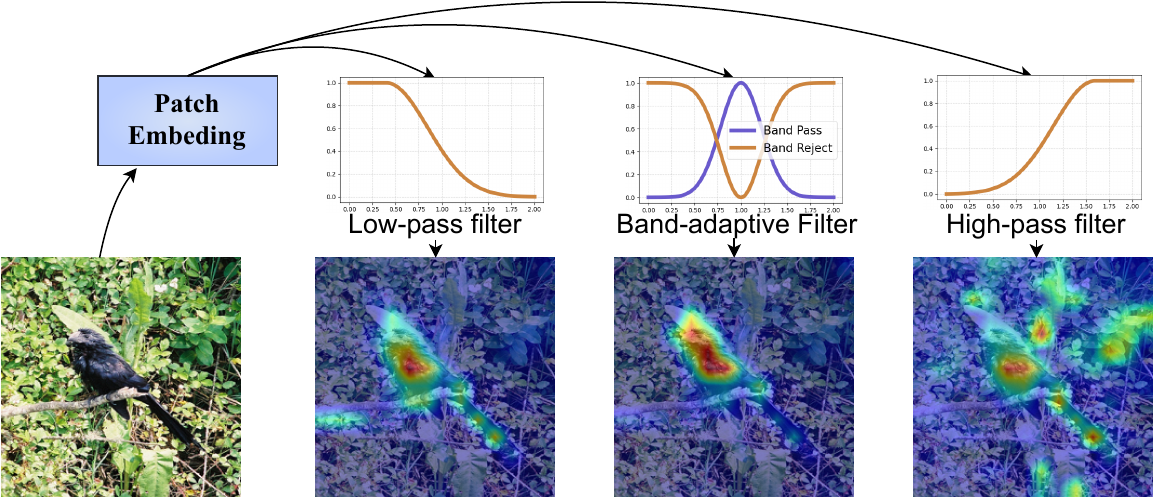}%
\vspace{-2ex}
\caption{Heatmaps using different filters. Heatmaps taken from the \textit{values} matrix of filter outputs trained on the CUB-200 dataset \cite{cub} using the low-pass filter, high-pass filter, and the proposed Band-adaptive Filter (BaF), respectively.}\label{fig2}
\vspace{-3ex}
\end{figure}

The spectral analysis presented above provides a solid theoretical foundation for understanding how token-wise correlations in ViTs can be reinterpreted through graph-based and Fourier-based operations. This analysis not only offers a mathematical perspective on the attention mechanism but also informs the design of practical spectral filtering modules. In the following section, we will elaborate on how these theoretical insights guide the development of two complementary branches—the implicit and explicit spectral branches—and how they are seamlessly integrated into the core architectural components of the proposed SIEFormer, directly corresponding to the discussed spectral perspectives.

\subsection{Spectral-Interpretable and -Enhanced Transformer}
By revisiting and rethinking the Vision Transformer (ViT), we can rebuild an enhanced Transformer architecture that is interpretable from a spectral perspective. As illustrated in Fig. \ref{fig_main}, the SIEFormer comprises two key parallel branches: an implicit branch coupled with the Band-adaptive Filter (BaF) layer, which focuses on learning correlations within the local neighborhood of tokens, and an explicit spectral branch underpinned by the Maneuverable Filtering Layer (MFL), which concentrates on capturing global dependencies among tokens. The following sections elaborate on the design and functionality of these two complementary branches. \\  
\noindent\textbf{Band-adaptive Filter (BaF)}
Filters generally exhibit one or more characteristics of the low-pass or high-pass \cite{Beyond}. When visualizing the responses of these types of filters on feature maps, it becomes apparent that high-frequency signals often correspond to variable background details, while low-frequency signals predominantly capture the main object of interest, as shown in Fig. \ref{fig2}. This observation serves as inspiration for the development of a novel filtering approach capable of harnessing the advantages of different filters. To avoid the heavy computational cost of the higher-order graph Laplacian matrices $\Tilde{\mathbf{L}}$, the previously mentioned filter functions are written in linear and quadratic form as     
\begin{equation}
    \label{ci}
    \begin{aligned}
        g_{l}&=p(a\mathbf{\Lambda}-(1-a)\mathbf{I}),\\
        g_{h}&=p((1-a)\mathbf{I}-a\mathbf{\Lambda}),
    \end{aligned}
\end{equation}
where $g_{l}$ and $g_{h}$ represent low-pass and high-pass filters, parameter $p>0$ controls the magnitude of frequency profile, parameter $a \in (0, 1)$ determines the cut-off frequency \cite{Beyond}, and $\mathbf{\Lambda}$ represents all eigenvalues of normalized Laplacian matrix $\Tilde{\mathbf{L}}$. 

Building upon the analysis of the aforementioned filters, the Band-adaptive filter is proposed to complementarily filter signals such that the filter works selectivity between the band-pass and band-reject filters. Formally, the designed filters are expressed as 
\begin{equation}
\label{bp}
\begin{aligned}
g_{p}(\mathbf{\Lambda})&=\mathbf{\Lambda}^2-\mathbf{I},\\
g_{r}(\mathbf{\Lambda})&=\mathbf{\Lambda}^2,
\end{aligned}
\end{equation}
where $g_p$ and $g_r$ stand for band-pass and band-reject filters, respectively. The Band-adaptive Filtering operation of graph signals can then be written as
\begin{equation}
\label{POL}
\hat{\mathbf{V}}=\mathbf{U}\sum_{*\in \{p, r\}}g_{\theta}^{*}(\mathbf{\Lambda})\mathbf{U}^T\mathbf{V}.
\end{equation}

The proposed filter can either enhance or suppress the specific frequencies of the signal by automatically accommodating the parameter $\theta$. Introducing the Band-adaptive Filter to the complete form of the $1^{st}$-order Chebyshev approximation in Eq.\eqref{1cheb} , resulting in 
\begin{equation}
\label{BaF}
\mathcal{S}_{BaF}(\mathbf{T}, \mathbf{V})= \mathbf{T}^2\mathbf{V}\mathbf{W}_{P} + (\mathbf{T}^2-\mathbf{I}_N)\mathbf{V}\mathbf{W}_{R},
\end{equation}
where $\mathbf{W}_{P}$ and $\mathbf{W}_{R}$ are the parameter matrices of shape of $\mathbb{R}^{C\times C}$. In order to stabilise the output of the network to avoid excessive losses that would lead to slow convergence of the model, and to make better use of the pre-training parameters to improve the generalization of the model, $\mathbf{W}_{P}$ and $\mathbf{W}_{R}$ are initialized to a zero matrix. In addition, it presents the following Proposition to better reveal the essence of adopting the Band-adaptive Filter.

\begin{prop}
The symmetrically normalized adjacency matrix \(\tilde{\mathbf{A}}\) is a reasonable simplification of \(\mathbf{T}\) in Eq. (\ref{BaF}).
\end{prop}
\begin{proof}
Let \(\lambda_{min} \leq \cdots \leq \lambda_{max}\) be the eigenvalue of the normalized Laplacian matrix \(\tilde{\mathbf{L}}\), the declaration is equivalent to confirm that the eigenvalue bounds of \(\tilde{\mathbf{L}}\) lie in \([0, 2]\).  

Given the fact that the Laplacian matrix \(\mathbf{L}\) is symmetric and diagonally dominant, it is positive-semidefinite\footnote{\url{https://en.wikipedia.org/wiki/Laplacian_matrix}}. If let \(\mathbf{B}\) be an incidence matrix of an orientation of the graph \(\mathcal{G}\), then \(\mathbf{L}=\mathbf{B}\mathbf{B}^{\operatorname{T}}\), thus \(\textbf{\emph{f}}^{\operatorname{T}}\mathbf{L}\textbf{\emph{f}}=\mathbf{B}\textbf{\emph{f}}^{2} \geq 0\) for \( \forall \textbf{\emph{f}} \in \mathbb{R}^{N}\), and it can have \(\tilde{\mathbf{L}}=\mathbf{V}\mathbf{V}^{\operatorname{T}}\) for \(\mathbf{V}=\mathbf{D}^{-\frac{1}{2}}\mathbf{B}\) which means that \(\tilde{\mathbf{L}}\) is also a positive-semidefinite matrix as for any \( \textbf{\emph{f}} \in \mathbb{R}^{N}\):  
\[
    \begin{split}
    \textbf{\emph{f}}^{\operatorname{T}}\tilde{\mathbf{L}}\textbf{\emph{f}} & =\textbf{\emph{f}}^{\operatorname{T}}\left(\mathbf{I}-\mathbf{\tilde{A}}\right)\textbf{\emph{f}} \\
    & = \sum_{i \in \mathbf{V}}
    \emph{f}\left(i\right)^{2}-\sum_{(i,j) \in \mathbf{E}}\frac{2\emph{f}\left(i\right)\emph{f}\left(j\right)}{\sqrt{d\left(i\right)d\left(j\right)}}\\
    & = \sum_{(i,j) \in \mathbf{E}}\left(\frac{\emph{f}\left(i\right)}{\sqrt{d\left(i\right)}}-\frac{\emph{f}\left(j\right)}{\sqrt{d\left(j\right)}}\right)^{2} \geq 0,
    \end{split}
\]
where \(d\left(i\right)\) is the \(i^{th}\) element of the degree matrix \(\mathbf{D}\), and the result indicates that \(\tilde{\mathbf{L}}\) has nonnegative value, namely, \(\lambda_{min} = 0\). Here, the graph vertices and edges are represented as \(\mathbf{V}\) and \(\mathbf{E}\) for readability. 

Considering the positive semi-definiteness of \(\mathbf{\tilde{L}}\), it implies
\(-\textbf{\emph{f}}^{\operatorname{T}}\mathbf{\tilde{A}\textbf{\emph{f}}} \leq \textbf{\emph{f}}^{\operatorname{T}}\textbf{\emph{f}} \Rightarrow \textbf{\emph{f}}^{\operatorname{T}}\left(\mathbf{I}-\mathbf{\tilde{A}}\right)\textbf{\emph{f}} \leq 2\textbf{\emph{f}}^{\operatorname{T}}\textbf{\emph{f}} \Rightarrow \frac{\textbf{\emph{f}}^{\operatorname{T}}\mathbf{\tilde{L}\textbf{\emph{f}}}}{\textbf{\emph{f}}^{\operatorname{T}}\textbf{\emph{f}}} \leq 2.\) Then, the upper bound can be derived using the Rayleigh quotient\footnote{\url{https://en.wikipedia.org/wiki/Rayleigh_quotient}}, namely, \(\lambda_{max} \leq 2.\)

Finally, the bounds of the resulting \(\lambda\) can be substituted into Eq. (\ref{BaF}) to give \(\mathbf{T}=\mathbf{I}_{N}-\tilde{\mathbf{L}}=\mathbf{\tilde{A}}\).
\end{proof}

Replacing the parameter $\mathbf{T}$ with $\Tilde{\mathbf{A}}$ in Eq. \eqref{BaF}, resulting in the spectral view of the proposed Transformer that uses the BaF filter which can be formulated as
\begin{equation}
\label{BF1}
    \mathcal{S}_{BaF}(\Tilde{\mathbf{A}}, \mathbf{V})= \Tilde{\mathbf{A}}^2\mathbf{V}\mathbf{W}_{P} + (\Tilde{\mathbf{A}}^2-\mathbf{I}_N)\mathbf{V}\mathbf{W}_{R}.
\end{equation}

Since the \textit{values} matrix in the self-attention mechanism is treated as node signals, spectral domain graph filtering usually considers signals from neighbouring nodes, thus it might hinder the ability to capture the long-term dependencies of features in the frequency domain. To cope with this issue, we propose the Maneuverable filtering layer that realises the processing of \textit{value} matrix to automatically filter out the noise in the \textit{value} matrix detached from the graph adjacency matrix.

\begin{algorithm}[!t] \label{alg1}
\caption{SIEFormer attention mechanism in Python style}
\label{algorithm: SIEFormera}
\quad \textcolor{gre}{\# Input: Feature $x$}

\quad \textcolor{gre}{\# Output: Fused feature $y$}

\quad \textcolor{gre}{\# Initialization: Filter parameter $W_f$ intialized with $0$, MFL parameter $W_M$ inintialized with $1+0i$.}\\

\quad $W_f$.require\_grad=True \textcolor{gre}{\#  unfreeze the parameter of filter}

\quad $W_M$.require\_grad=True \textcolor{gre}{\#  unfreeze the parameter of MFL}\\

\quad Q, K, V = qkv(x) \textcolor{gre}{\#  Q, K, V in shape of (h, n, c)}

\quad Q$_f$, K$_f$, V$_f$ = qkv(x).reshape(n, c$\times$h) \textcolor{gre}{\#  query, key, value for each patch without multi-head}

\quad V$_l$ = ifft($f$(fft(V), $\theta_M$)) \textcolor{gre}{\#  Maneuverable Filtering Layer}

\quad exp\_spe\_out = multi-head-attention(Q, K, V$_l$) \textcolor{gre}{\#  explicit spectral-view output with MFL}\\

\quad attn = Q$_f$@K$_f^T$ \textcolor{gre}{\#  attention without multi-head}

\quad attn\_hat = ReLU((attn + attn$^T$)/2) \textcolor{gre}{\#  make the matrix symmetric}

\quad L = I - Norm(attn\_hat) \textcolor{gre}{\#  Laplacian matrix}

\quad imp\_spe\_out = $f$(L, V$_f$, $W_f$) \textcolor{gre}{\#  implicit spectral-view output with BaF}\\

\quad y = exp\_spe\_out + imp\_spe\_out \textcolor{gre}{\#  fused feature}
\end{algorithm}

\noindent\textbf{Maneuverable Filtering Layer (MFL)}
The processing of \textit{values} using graph filter is based on the Laplace matrix obtained from the graph adjacency matrix, which means the processing of the signal is highly correlated with the similarity between each node. However, the spectral aware features of \textit{values} that are detached from the graph adjacency matrix with truncated approximations (local neighborhoods of tokens) are also of interest, which can be beneficial in capturing the deeper information of the spectral domain (global dependencies of tokens). Thus, we propose the MFL for the explicit use of spectral domain features (Explicit Spectral Branch in Fig. \ref{fig_main}). Concretely, for a given \textit{value} matrix $\mathbf{V}\in \mathbb{R}^{N\times C_e}$, we first perform a 2D Fourier Transform $\mathcal{F}$ on to transform it into the frequency domain:
\begin{equation}
    \label{ft1}
    \Bar{\mathbf{V}}=\mathcal{F}(\mathbf{V})\in \mathbb{R}^{N\times C_e},
\end{equation}
where $\Bar{\mathbf{V}}$ represents spectrum of $\mathbf{V}$ and is a complex tensor. By introducing a complex matrix parameter $\mathbf{W}_C$, we can automatically modulate the spectrum and filter out the noise, which can be expressed as
\begin{equation}
    \label{ft2}
    \mathbf{V}_{R}=\mathcal{F}^{-1}(\Bar{\mathbf{V}}\mathbf{W}_C)\mathbf{W}_Z,
\end{equation}
where $\mathcal{F}^{-1}$ represents inverse Fourier transform ,$\mathbf{V}_{R}$ denotes the reconstructed \textit{value} matrix $\mathbf{V}$, $\mathbf{W}_Z$ represents the pre-training weights of the original self-attention mechanism. It is worth noting that the initialization of $\mathbf{W}_Z$ is set to $1+0i$ and is prohibited from updating during training to reduce the bias caused by modifying the pre-trained backbone.

\noindent\textbf{SIEFormer}
The proposed SIEFormer is composed of three main components, including two innovative spectral views based on the explicit use of the Fourier transform within MFL and the implicit use of the graph Fourier transform within BaF, as well as the standard residual connection between the input and output features (the Python style pseudo-code can be seen in Algorithm \ref{algorithm: SIEFormera}), which can be denoted as
\begin{equation}
    \label{Fa}
    \mathbf{Y}= \mathbf{X} + \mathbf{M}\mathbf{V}_{R}+ \Tilde{\mathbf{A}}^2\mathbf{V}\mathbf{W}_{P} + (\Tilde{\mathbf{A}}^2-\mathbf{I}_N)\mathbf{V}\mathbf{W}_{R}.
\end{equation}

\begin{table*}[!t]
\centering
\caption{Statistical comparison of data partitions (i.e., labeled and unlabeled) across different datasets.}
\label{tab1}
\vspace{-2ex}
\begin{tabular}{lccccccccc}
\toprule
\multicolumn{2}{c}{} & \multicolumn{1}{c}{CIFAR-10} & \multicolumn{1}{c}{CIFAR-100} & \multicolumn{1}{c}{ImageNet-100} & \multicolumn{1}{c}{CUB-200} & \multicolumn{1}{c}{Stanford-Cars} & \multicolumn{1}{c}{FGVC-Aircraft} & \multicolumn{1}{c}{Herbarium19} & \multicolumn{1}{c}{ImageNet-1k}\\ \cmidrule(l){1-10} 
\multirow{2}{*}{Labeled} & \multicolumn{1}{c}{$\vert \mathcal{Y}_{l}\vert$} & \multicolumn{1}{c}{5} & \multicolumn{1}{c}{80} & \multicolumn{1}{c}{50} & \multicolumn{1}{c}{100} & \multicolumn{1}{c}{98} & \multicolumn{1}{c}{50} & \multicolumn{1}{c}{341} & \multicolumn{1}{c}{500}\\
& \multicolumn{1}{c}{$\vert\mathcal{D}_l\vert$} & \multicolumn{1}{c}{12.5k} & \multicolumn{1}{c}{20k} & \multicolumn{1}{c}{31.9k} & \multicolumn{1}{c}{1.5k} & \multicolumn{1}{c}{2.0k} & \multicolumn{1}{c}{1.7k} & \multicolumn{1}{c}{8.9k} & \multicolumn{1}{c}{321k}\\ \cmidrule(l){1-10} 
\multirow{2}{*}{Unlabeled} & \multicolumn{1}{c}{$\vert \mathcal{Y}_u\vert$} & \multicolumn{1}{c}{10} & \multicolumn{1}{c}{100} & \multicolumn{1}{c}{100} & \multicolumn{1}{c}{200} & \multicolumn{1}{c}{196} & \multicolumn{1}{c}{100} & \multicolumn{1}{c}{683} & \multicolumn{1}{c}{1000}\\
& \multicolumn{1}{c}{$\vert\mathcal{D}_u\vert$} & \multicolumn{1}{c}{37.5k} & \multicolumn{1}{c}{30k} & \multicolumn{1}{c}{95.3k} & \multicolumn{1}{c}{4.5k} & \multicolumn{1}{c}{6.1k} & \multicolumn{1}{c}{5.0k} & \multicolumn{1}{c}{25.4k} & \multicolumn{1}{c}{960k}\\
\bottomrule
\end{tabular}
\vspace{-2ex}
\end{table*}

\begin{table*}[!t]
\centering
\caption{Results on three generic datasets. `-’ means not reported.}
\label{generic}
\begin{tabular}{lccccccccc}
\toprule
\multicolumn{1}{l}{} & \multicolumn{3}{c}{CIFAR-10} & \multicolumn{3}{c}{CIFAR-100} & \multicolumn{3}{c}{ImageNet-100} \\ \cmidrule(l){2-10} 
\multicolumn{1}{l}{Methods} & All & Old & \multicolumn{1}{c|}{New} & All & Old & \multicolumn{1}{c|}{New} & All & Old & New \\ \cmidrule(l){1-10} 
\multicolumn{1}{l}{RankStats+ \cite{autonovel1}} & 46.8 & 19.2 & \multicolumn{1}{c|}{60.5} & 58.2 & 77.6 & \multicolumn{1}{c|}{19.3} & 37.1 & 61.6 & 24.8 \\
\multicolumn{1}{l}{UNO+ \cite{UNO}} & 68.6 & \textbf{98.3} & \multicolumn{1}{c|}{53.8} & 69.5 & 80.6 & \multicolumn{1}{c|}{47.2} & 70.3 & 95.0 & 57.9 \\
\multicolumn{1}{l}{ORCA \cite{orca}} & 81.8 & 86.2 & \multicolumn{1}{c|}{79.6} & 69.0 & 77.4 & \multicolumn{1}{c|}{52.0} & 73.5 & 92.6 & 63.9 \\
\multicolumn{1}{l}{GCD \cite{GCD}} & 91.5 & 97.9 & \multicolumn{1}{c|}{88.2} & 73.0 & 76.2 & \multicolumn{1}{c|}{66.5} & 74.1 & 89.8 & 66.3 \\
\multicolumn{1}{l}{DCCL \cite{DCCL}} & 96.3 & 96.5 & \multicolumn{1}{c|}{96.9} & 75.3 & 76.8 & \multicolumn{1}{c|}{70.2} & 80.5 & 90.5 & 76.2 \\
\multicolumn{1}{l}{SimGCD \cite{wen2023simgcd}} & 97.1 & 95.1 & \multicolumn{1}{c|}{98.1} & 80.1 & 81.2 & \multicolumn{1}{c|}{77.8} & 83.0 & 93.1 & 77.9 \\
\multicolumn{1}{l}{PromptCAL \cite{promptcal}} & \textbf{97.9} & 96.6 & \multicolumn{1}{c|}{98.5} & 81.2 & 84.2 & \multicolumn{1}{c|}{75.3} & 83.1 & 92.7 & 78.3 \\
\multicolumn{1}{l}{CMS \cite{CMS}} & \textbf{-} & - & \multicolumn{1}{c|}{-} & 82.3 & \textbf{85.7} & \multicolumn{1}{c|}{75.5} & 84.7 & \textbf{95.6} & 79.2 \\
\multicolumn{1}{l}{SPTNet \cite{wang2024sptnet}} & 97.3 & 95.0 & \multicolumn{1}{c|}{\textbf{98.6}} & 81.4 & 84.3 & \multicolumn{1}{c|}{75.6} & 85.4 & 93.2 & 81.4 \\
\multicolumn{1}{l}{HypGCD \cite{hyper2025TNNLS}} & 96.6 & 95.5 & \multicolumn{1}{c|}{97.2} & 79.3 & 80.1 & \multicolumn{1}{c|}{77.8} & - & - & - \\
\multicolumn{1}{l}{ProtoGCD \cite{ma2025protogcd}} & 97.3 & 95.3 & \multicolumn{1}{c|}{98.2} & 81.9 & 82.9 & \multicolumn{1}{c|}{80.0} & 84.0 & 92.2 & 79.9 \\
\midrule
\multicolumn{1}{l}{\textbf{SIEFormer (GCD)}} & 97.3$^{\green{\text{+5.8}}}$ & 95.0$^{\red{\text{-2.9}}}$ & \multicolumn{1}{c|}{98.4$^{\green{\text{+10.2}}}$} & 77.1$^{\green{\text{+4.1}}}$ & 81.5$^{\green{\text{+5.3}}}$ & \multicolumn{1}{c|}{68.3$^{\green{\text{+1.8}}}$} & 79.9$^{\green{\text{+5.8}}}$ & 91.0$^{\green{\text{+1.2}}}$ &  74.3$^{\green{\text{+8.0}}}$\\
\multicolumn{1}{l}{\textbf{SIEFormer (SimGCD)}} & 97.4$^{\green{\text{+0.3}}}$ & 95.2$^{\green{\text{+0.1}}}$ & \multicolumn{1}{c|}{\textbf{98.6}$^{\green{\text{+0.5}}}$} & \textbf{83.3}$^{\green{\text{+3.2}}}$ & 84.4$^{\green{\text{+3.2}}}$ & \multicolumn{1}{c|}{\textbf{81.2}$^{\green{\text{+3.4}}}$} & \textbf{86.5}$^{\green{\text{+3.5}}}$ & 94.5$^{\green{\text{+1.4}}}$ & \textbf{82.6}$^{\green{\text{+4.7}}}$ \\
\multicolumn{1}{l}{\textbf{SIEFormer (SPTNet)}} & 97.4$^{\green{\text{+0.1}}}$ & 95.1$^{\green{\text{+0.1}}}$ & \multicolumn{1}{c|}{98.6$^{\green{\text{+0.0}}}$} & 82.6$^{\green{\text{+1.2}}}$ & 85.4$^{\green{\text{+1.1}}}$ & \multicolumn{1}{c|}{77.0$^{\green{\text{+1.4}}}$} & 86.3$^{\green{\text{+0.9}}}$ & 94.2$^{\green{\text{+1.0}}}$ &  82.3$^{\green{\text{+0.9}}}$\\
\bottomrule
\end{tabular}
\vspace{-2ex}
\end{table*}

\section{Experiments}\label{sec_4}

\subsection{Datasets and Evaluation Protocol}
The proposed SIEFormer was evaluated on several benchmark datasets for image recognition tasks. These include the CIFAR-10 and CIFAR-100 \cite{cifar} datasets, the ImageNet-100 dataset \cite{imagenet}), and the recently introduced Semantic Shift Benchmark datasets \cite{vaze2022openset} which consist of the CUB-200 \cite{cub}, Stanford-Cars \cite{scar}, and FGVC-Aircraft \cite{aircraft}). Additionally, the more challenging Herbarium 2019 fine-grained classification dataset \cite{herbarium-2019-fgvc6} was also used for evaluation.

\noindent\textbf{CIFAR-10 \& -100} \cite{cifar} are two datasets containing 10 and 100 categories respectively, covering a wide range of objects in our daily lives, such as cats, dogs, etc. CIFAR-10 contains 60,000 32x32 color images, while CIFAR-100 contains richer categories, with the same number of images and sizes.

\noindent\textbf{ImageNet-1K \& -100} \cite{imagenet} is an ImageNet subset dataset containing 1000 categories and 100 categories. Each category has hundreds to thousands of images. ImageNet is commonly used to evaluate the performance and generalization ability of models on image classification tasks.

\noindent\textbf{CUB-200} \cite{cub} is a widely-used fine-grained image classification dataset, consisting of 11,788 images of 200 different bird species. Each image is meticulously annotated with detailed attribute information, making it a challenging benchmark for evaluating the ability of models to recognize subtle visual differences between closely related classes. 

\noindent\textbf{Stanford Cars} \cite{scar} is a well-established fine-grained vehicle classification dataset. It contains 8,144 images spanning 196 different car models. The dataset captures the intricate visual details and stylistic variations that characterize different makes and models.

\noindent\textbf{Herbarium 2019} \cite{herbarium-2019-fgvc6} is a large-scale fine-grained image dataset for plant species identification. It consists of over 46,000 high-resolution images of herbarium specimens, representing 680 different species of the Melastomataceae plant family. 

\noindent\textbf{FGVC-Aircraft} \cite{aircraft} is a dataset contains approximately 10,000 images of 100 different aircraft models, designed to advance research in aircraft recognition and classification.

For the experimental evaluation, the datasets were partitioned into labeled and unlabeled subsets. Specifically, for the CIFAR-100 dataset, 80\% of the data was allocated to the labeled set $\mathcal{D}_l$, while for the remaining datasets, this ratio was set to 50\% (see Table \ref{tab1}). The labeled data was sampled from the annotated classes, and the unlabeled data $\mathcal{D}_u$ was formed by merging the remaining instances.

Following the evaluation protocol established by prior work (GCD), the cluster accuracy was assessed by assigning the predicted labels to the ground truth labels. The reported accuracy metrics include the performance on the labeled data, the unlabeled data, as well as the overall accuracy across all classes.

\subsection{Implementation Details}
In line with the GCD setting, the backbone employed in our experiments was ViT-B/16 \cite{ViT}, which had been pre-trained on the Imagenet dataset using the DINO self-supervised learning approach \cite{DINO}. While most prior works have only fine-tuned the last block of the ViT model, the proposed approach replaced some of the attention blocks with the designed SIEFormer architecture and fine-tuned the filter parameters. All experiments were conducted using an NVIDIA GeForce RTX 3090 GPU.

\noindent\textbf{Details on SIEFormer on GCD.} The learning rate was set to 0.0005 and the batch size was set to 128, and the other hyper-parameters remain unchanged.

\noindent\textbf{Details on SIEFormer on SimGCD.} The batch size was set to 128 and trained for 200 epochs. Following SimGCD, supervised contrastive learning weight $\lambda$ is set to 0.35, and the temperature values $\tau_{u}$, $\tau_c$  as 0.07, 1.0, respectively. For the classification objective, we also use the same hyper-parameters as that in original SimGCD, which set $\tau_s$ to 0.1, and $\tau_t$ was initialized to 0.07, then warmed up to 0.04 with a cosine schedule for the first 30 epochs.

\noindent\textbf{Details on SIEFormer on SPTNet.} The parameters used to train our SIEFormer (SPTNet) are identical to those in the original SPTNet \cite{wang2024sptnet}. More specifically, the spatial prompt size $m$ and the global prompt size are set to 1 and 30, respectively. Furthermore, the two stages alternate every k = 20 iterations. All prompts are trained for 1,000 epochs with a batch size of 128.

\begin{table*}[!t]
\centering
\caption{Results on four fine-grained datasets. `-' means not reported.}
\label{finegrained}
\resizebox{\textwidth}{!}{
\begin{tabular}{ccccccccccccc}
\toprule
\multicolumn{1}{l}{} & \multicolumn{3}{c}{CUB-200} & \multicolumn{3}{c}{Stanford-Cars} & \multicolumn{3}{c}{FGVC-Aircraft} & \multicolumn{3}{c}{Herbarium19} \\ \cmidrule(l){2-13} 
\multicolumn{1}{l}{Methods} & All & Old & \multicolumn{1}{c|}{New} & All & Old & \multicolumn{1}{c|}{New} & All & Old & \multicolumn{1}{c|}{New} & All & Old & New \\ \cmidrule(l){1-13} 
\multicolumn{1}{l}{RankStats+ \cite{autonovel1}} & 33.3 & 51.6 & \multicolumn{1}{c|}{24.2} & 28.3 & 61.8 & \multicolumn{1}{c|}{12.1} & 26.9 & 36.4 & \multicolumn{1}{c|}{22.2} & 27.9 & 55.8 & 12.8 \\
\multicolumn{1}{l}{UNO+ \cite{UNO}} & 35.1 & 49.0 & \multicolumn{1}{c|}{28.1} & 35.5 & 70.5 & \multicolumn{1}{c|}{18.6} & 40.3 & 56.4 & \multicolumn{1}{c|}{32.2} & 28.3 & 53.7 & 14.7 \\
\multicolumn{1}{l}{ORCA \cite{orca}} & 35.3 & 45.6 & \multicolumn{1}{c|}{30.2} & 23.5 & 50.1 & \multicolumn{1}{c|}{10.7} & 20.9 & 30.9 & \multicolumn{1}{c|}{11.5} & 22.0 & 31.8 & 17.1 \\
\multicolumn{1}{l}{GCD \cite{GCD}} & 51.3 & 56.6 & \multicolumn{1}{c|}{48.7} & 39.0 & 57.6 & \multicolumn{1}{c|}{29.9} & 45.0 & 41.1 & \multicolumn{1}{c|}{46.9} & 35.4 & 51.0 & 27.0 \\
\multicolumn{1}{l}{DCCL \cite{DCCL}} & 63.5 & 60.8 & \multicolumn{1}{c|}{64.9} & 43.1 & 55.7 & \multicolumn{1}{c|}{36.2} & \multicolumn{1}{c}{-} & \multicolumn{1}{c}{-} & \multicolumn{1}{c|}{-} & - & - & - \\
\multicolumn{1}{l}{SimGCD \cite{wen2023simgcd}} & 60.3 & 65.6 & \multicolumn{1}{c|}{57.7} & 53.8 & 71.9 & \multicolumn{1}{c|}{45.0} & \multicolumn{1}{c}{54.2} & \multicolumn{1}{c}{59.1} & \multicolumn{1}{c|}{51.8} & 43.0 & 58.0 & 35.1 \\
\multicolumn{1}{l}{PromptCAL \cite{promptcal}} & 62.9 & 64.4 & \multicolumn{1}{c|}{62.1} & 50.2 & 70.1 & \multicolumn{1}{c|}{40.6} & \multicolumn{1}{c}{52.2} & \multicolumn{1}{c}{52.2} & \multicolumn{1}{c|}{52.3} & 37.0 & 52.0 & 28.9 \\
\multicolumn{1}{l}{CMS \cite{CMS}} & 68.2 & \textbf{76.5} & \multicolumn{1}{c|}{64.0} & 56.9 & 76.1 & \multicolumn{1}{c|}{47.6} & \multicolumn{1}{c}{56.0} & \multicolumn{1}{c}{\textbf{63.4}} & \multicolumn{1}{c|}{52.3} & 36.4 & 54.9 & 26.4 \\
\multicolumn{1}{l}{$\mu$GCD \cite{vaze2023clevr4}} & 65.7 & 68.0 & \multicolumn{1}{c|}{64.6} & 56.5 & 68.1 & \multicolumn{1}{c|}{50.9} & \multicolumn{1}{c}{53.8} & \multicolumn{1}{c}{55.4} & \multicolumn{1}{c|}{53.0} & \textbf{45.8} & \textbf{61.9} & 37.2 \\
\multicolumn{1}{l}{SPTNet \cite{wang2024sptnet}} & 65.8 & 68.8 & \multicolumn{1}{c|}{65.1} & 59.0 & \textbf{79.2} & \multicolumn{1}{c|}{49.3} & \multicolumn{1}{c}{59.3} & \multicolumn{1}{c}{61.8} & \multicolumn{1}{c|}{58.1} & 43.4 & 58.7 & 35.2 \\
\multicolumn{1}{l}{HypGCD \cite{hyper2025TNNLS}} & 65.8 & 75.1 & \multicolumn{1}{c|}{61.1} & 57.6 & 74.9 & \multicolumn{1}{c|}{49.2} & \multicolumn{1}{c}{54.8} & \multicolumn{1}{c}{63.5} & \multicolumn{1}{c|}{50.6} & - & - & - \\
\multicolumn{1}{l}{ProtoGCD \cite{ma2025protogcd}} & 63.2 & 68.5 & \multicolumn{1}{c|}{60.5} & 53.8 & 73.7 & \multicolumn{1}{c|}{44.2} & \multicolumn{1}{c}{56.8} & \multicolumn{1}{c}{62.5} & \multicolumn{1}{c|}{\textbf{53.9}} & 44.5 & 59.4 & 36.5 \\\midrule

\multicolumn{1}{l}{\textbf{SIEFormer (GCD)}} & 54.6$^{\green{\text{+3.3}}}$ & 57.9$^{\green{\text{+1.3}}}$ & \multicolumn{1}{c|}{53.0$^{\green{\text{+4.3}}}$} & 44.3$^{\green{\text{+5.3}}}$ & 62.4$^{\green{\text{+5.2}}}$ & \multicolumn{1}{c|}{35.6$^{\green{\text{+5.7}}}$} & 47.4$^{\green{\text{+2.4}}}$ & 47.6$^{\green{\text{+6.5}}}$ & \multicolumn{1}{c|}{47.4$^{\green{\text{+0.5}}}$} & 36.5$^{\green{\text{+1.1}}}$ & 54.5$^{\green{\text{+3.5}}}$ & 26.8$^{\red{\text{-0.2}}}$ \\
\multicolumn{1}{l}{\textbf{SIEFormer (SimGCD)}} & 66.0$^{\green{\text{+5.7}}}$ & 71.6$^{\green{\text{+6.0}}}$ & \multicolumn{1}{c|}{63.2$^{\green{\text{+5.5}}}$} & 59.2$^{\green{\text{+5.4}}}$ & 74.0$^{\green{\text{+2.1}}}$ & \multicolumn{1}{c|}{\textbf{52.1$^{\green{\text{+7.1}}}$}} & 53.3$^{\red{\text{-0.9}}}$ & 60.2$^{\green{\text{+1.1}}}$ & \multicolumn{1}{c|}{49.8$^{\red{\text{-2.0}}}$} & 44.5$^{\green{\text{+1.5}}}$ & 57.2$^{\red{\text{-0.8}}}$ & 37.6$^{\green{\text{+2.5}}}$ \\
\multicolumn{1}{l}{\textbf{SIEFormer (SPTNet)}} & \textbf{69.1}$^{\green{\text{+3.3}}}$ & 75.1$^{\green{\text{+6.3}}}$ & \multicolumn{1}{c|}{\textbf{66.2}$^{\green{\text{+1.1}}}$} & \textbf{60.1}$^{\green{\text{+1.1}}}$ & 77.8$^{\red{\text{-1.4}}}$ & \multicolumn{1}{c|}{51.5$^{\green{\text{+2.2}}}$} & \textbf{59.4}$^{\green{\text{+0.1}}}$ & 62.5$^{\green{\text{+0.7}}}$ & \multicolumn{1}{c|}{57.8$^{\red{\text{-0.3}}}$} & 44.7$^{\green{\text{+1.3}}}$ & 57.7$^{\red{\text{-1.0}}}$ & \textbf{37.7}$^{\green{\text{+2.5}}}$\\
\bottomrule
\end{tabular}}
\vspace{-3ex}
\end{table*}

\begin{table}[!t]
\centering
\caption{Results on ImageNet-1K. '$^*$' indicates that we implemented it directly with the code provided due to the original article didn't report it.}
\label{tab_img1k}
\vspace{-2ex}
\begin{tabular}
{cccc}
\toprule
\multicolumn{1}{c}{} & \multicolumn{3}{c}{ImageNet-1K} \\ \cmidrule(l){2-4}

\multicolumn{1}{l}{Methods} & \multicolumn{1}{c}{All} & \multicolumn{1}{c}{Old} & \multicolumn{1}{c}{New} \\ \cmidrule(l){1-4}

\multicolumn{1}{l|}{GCD} & \multicolumn{1}{c}{52.5} & \multicolumn{1}{c}{72.5} & \multicolumn{1}{c}{42.2} \\

\multicolumn{1}{l|}{SimGCD} & \multicolumn{1}{c}{57.1} & \multicolumn{1}{c}{77.3} & \multicolumn{1}{c}{46.9}\\ 

\multicolumn{1}{l|}{SPTNet$^*$} & \multicolumn{1}{c}{56.9} & \multicolumn{1}{c}{76.1} & \multicolumn{1}{c}{47.1}\\

\multicolumn{1}{l|}{HypGCD} & \multicolumn{1}{c}{55.0} & \multicolumn{1}{c}{75.7} & \multicolumn{1}{c}{44.6}\\
\midrule

\multicolumn{1}{l|}{\textbf{SIEFormer (SimGCD)}} & \multicolumn{1}{c}{\textbf{58.1}} & \multicolumn{1}{c}{78.0} & \multicolumn{1}{c}{\textbf{48.2}} \\
\multicolumn{1}{l|}{\textbf{SIEFormer (SPTNet)}} & \multicolumn{1}{c}{57.9} & \multicolumn{1}{c}{\textbf{78.5}} & \multicolumn{1}{c}{47.5} \\
\bottomrule
\end{tabular}
\vspace{-3ex}
\end{table}

\subsection{Main Results}
The proposed SIEFormer was evaluated across multiple datasets, with comprehensive comparisons made against state-of-the-art (SOTA) baseline models, including Rankstats+ \cite{autonovel1}, UNO \cite{UNO}, ORCA \cite{orca}, GCD \cite{GCD}, DCCL \cite{DCCL}, SimGCD \cite{wen2023simgcd}, PromptCAL \cite{promptcal}, $\mu$GCD \cite{vaze2023clevr4}, CMS \cite{CMS}, SPTNet \cite{wang2024sptnet}, ProtoGCD \cite{ma2025protogcd} and HypGCD \cite{hyper2025TNNLS}. Our model consistently demonstrate superior performance, achieving the new SOTA results.

\subsubsection{Performance on Generic Datasets}
We present a comparative analysis of various methods on generic datasets in Table \ref{generic}, with the support of SIEFormer, both commonly used baselines have shown significant improvements. Specifically, SimGCD with SIEFormer surpasses SPTNet \cite{wang2024sptnet} by notable margins, exhibiting performance improvements of 1.9\% on CIFAR-100, and 1.1\% on ImageNet-100 in terms of All classes. Furthermore, on New classes, SIEFormer demonstrates remarkable gains, surpassing 5.6\%, and 1.2\% on CIFAR-100, and ImageNet-100, respectively. Meanwhile, SPTNet with SIEFormer surpasses SPTNet \cite{wang2024sptnet} by 1.2\% on CIFAR-100, and 0.9\% on ImageNet-100 in terms of All classes. Notably, our fine-tuned backbone network under the baseline setting yields superior results on multiple datasets. Among existing methods, UNO+ exhibits robust performance on Old classes, particularly on CIFAR-10, while PromptCAL excels on CIFAR-10. However, when it comes to New classes, SIEFormer demonstrates superior stability, displaying a reduced bias towards the Old class across all datasets. 

\subsubsection{Performance on Fine-grained Datasets}
We further extended the evaluation of the SIEFormer to several fine-grained image classification datasets, with the results presented in Table \ref{finegrained}. Existing contrastive learning methods often face challenges in accurately classifying fine-grained images, as they require the model to effectively distinguish subtle visual differences between similar classes by capturing local patterns and contextual dependencies. Our SIEFormer introduces an innovative approach that yields substantial performance improvements. In particular, when compared to SimGCD, SimGCD with SIEFormer surpasses SimGCD by 5.7\% on CUB-200, 5.4\% on Stanford-cars, and 1.5\% on Herbarium19 for All classes. Also, SPTNet with SIEFormer showcases a margin improvement, surpassing SPTNet by 3.3\% on CUB-200, 1.1\% on Stanford-cars, and 1.3\% on Herbarium19 for All classes. This demonstrates that SIEFormer can achieve remarkable improvements when applied to different methods on fine-grained datasets.

\subsubsection{Performance on ImageNet-1K}

To comprehensively evaluate our method, we also conducted experiments on ImageNet-1K, a large-scale and more challenging generic classification dataset, and summarized the results In Table \ref{tab_img1k}. The experimental results show that SIEFormer exhibits consistent improvements across all metrics. This indicates the superiority of our SIEFormer when faced with more complex and realistic scenarios.

\begin{table}[!t]
\centering
\caption{Effects of each component on the datasets of Stanford-Cars and ImageNet-100.}
\label{comp}
\vspace{-2ex}
\begin{tabular}
{ccccccccc}
\toprule
\multicolumn{1}{c}{} & \multicolumn{2}{c}{Component} & \multicolumn{3}{c}{Stanford-Cars} & \multicolumn{3}{c}{ImageNet-100} \\ \cmidrule(l){1-9}
\multicolumn{1}{c}{} & \multicolumn{1}{c}{MFL} & \multicolumn{1}{c|}{BaF} & \multicolumn{1}{c}{All} & \multicolumn{1}{c}{Old} & \multicolumn{1}{c|}{New} & \multicolumn{1}{c}{All} & \multicolumn{1}{c}{Old} & \multicolumn{1}{c}{New} \\ \cmidrule(l){1-9}

\multicolumn{1}{c}{(1)} & \multicolumn{1}{c}{\XSolidBrush} & \multicolumn{1}{c|}{\XSolidBrush} & \multicolumn{1}{c}{53.8} & \multicolumn{1}{c}{71.9} & \multicolumn{1}{c|}{45.0} & \multicolumn{1}{c}{83.0} & \multicolumn{1}{c}{93.1} & \multicolumn{1}{c}{77.9} \\

\multicolumn{1}{c}{(2)} & \multicolumn{1}{c}{\Checkmark} & \multicolumn{1}{c|}{\XSolidBrush} & \multicolumn{1}{c}{38.6} & \multicolumn{1}{c}{54.6} & \multicolumn{1}{c|}{34.1} & \multicolumn{1}{c}{78.5} & \multicolumn{1}{c}{83.2} & \multicolumn{1}{c}{71.2} \\

\multicolumn{1}{c}{(3)} & \multicolumn{1}{c}{\XSolidBrush} & \multicolumn{1}{c|}{\Checkmark} & \multicolumn{1}{c}{57.8} & \multicolumn{1}{c}{77.8} & \multicolumn{1}{c|}{48.1} & \multicolumn{1}{c}{85.8} & \multicolumn{1}{c}{94.6} & \multicolumn{1}{c}{81.3}\\

\multicolumn{1}{c}{(4)} & \multicolumn{1}{c}{\Checkmark} & \multicolumn{1}{c|}{\Checkmark} & \multicolumn{1}{c}{59.2} & \multicolumn{1}{c}{74.0} & \multicolumn{1}{c|}{52.1} & \multicolumn{1}{c}{86.5} & \multicolumn{1}{c}{94.5} & \multicolumn{1}{c}{82.6} \\
\bottomrule
\end{tabular}
\vspace{-3ex}
\end{table}

\subsection{Ablation and Analysis}
In order to assess the effectiveness of SIEFormer, we conducted a series of ablation experiments. These experiments were designed to examine the impact of key factors within our model on its overall performance. For simplicity, we primarily evaluated the performance of SIEFormer on SimGCD.

\begin{table*}[!t]
\centering
\caption{Effects of removing extra bias on the filter.}
\label{bias}
\vspace{-2ex}
\begin{tabular}{p{0.3cm}p{0.3cm}p{0.3cm}p{0.3cm}p{0.3cm}p{0.3cm}p{0.3cm}p{0.3cm}p{0.3cm}p{0.3cm}p{0.3cm}p{0.3cm}p{0.3cm}p{0.3cm}p{0.3cm}p{0.3cm}p{0.3cm}p{0.3cm}p{0.3cm}}
\toprule
\multicolumn{1}{l}{} & \multicolumn{3}{c}{CIFAR-10} & \multicolumn{3}{c}{CIFAR-100} & \multicolumn{3}{c}{ImageNet-100} & \multicolumn{3}{c}{CUB-200} & \multicolumn{3}{c}{Stanford-Cars} & \multicolumn{3}{c}{Herbarium19} \\ \cmidrule(l){2-19} 
\multicolumn{1}{l}{Bias} & All & Old & \multicolumn{1}{c|}{New} & All & Old & \multicolumn{1}{c|}{New} & All & Old & \multicolumn{1}{c|}{New} & All & Old & \multicolumn{1}{c|}{New} & All & Old & \multicolumn{1}{c|}{New} & All & Old & New \\ \cmidrule(l){1-19} 
\multicolumn{1}{c}{\Checkmark} & 97.3 & 94.9 & \multicolumn{1}{c|}{98.5} & 81.4 & 84.7 & \multicolumn{1}{c|}{75.0} & 83.1 & 88.9 & \multicolumn{1}{c|}{77.8} & 64.9 & 73.2 & \multicolumn{1}{c|}{60.7} & 57.7 & 73.9 & \multicolumn{1}{c|}{49.8} & 43.2 & 60.0 & 34.1\\
\multicolumn{1}{c}{\XSolidBrush} & 97.4 & 95.2 & \multicolumn{1}{c|}{98.6} & 83.3 & 84.4& \multicolumn{1}{c|}{81.2} & 86.5 & 94.5 & \multicolumn{1}{c|}{82.6} & 66.0 & 71.6 & \multicolumn{1}{c|}{63.2} & 59.2 & 74.0 & \multicolumn{1}{c|}{52.1} & 44.5 & 57.2 & 37.6 \\ 
\bottomrule
\end{tabular}
\vspace{-2ex}
\end{table*}

\begin{table}[!t]
\centering
\caption{Effects of different filters on model performance.}
\label{Filter}
\vspace{-2ex}
\resizebox{0.48\textwidth}{!}{
\begin{tabular}
{cccccccc}
\toprule
\multicolumn{1}{c}{} & \multicolumn{1}{c}{} & \multicolumn{3}{c}{Stanford-Cars} & \multicolumn{3}{c}{ImageNet-100} \\ \cmidrule(l){3-8}

\multicolumn{1}{c}{} & \multicolumn{1}{l}{Filter} & \multicolumn{1}{c}{All} & \multicolumn{1}{c}{Old} & \multicolumn{1}{c|}{New} & \multicolumn{1}{c}{All} & \multicolumn{1}{c}{Old} & \multicolumn{1}{c}{New} \\ \cmidrule(l){1-8}

\multicolumn{1}{c}{(1)} & \multicolumn{1}{l|}{All-Pass Filter} & \multicolumn{1}{c}{32.4} & \multicolumn{1}{c}{56.6} & \multicolumn{1}{c|}{20.7} & \multicolumn{1}{c}{66.8} & \multicolumn{1}{c}{86.7} & \multicolumn{1}{c}{48.8} \\

\multicolumn{1}{c}{(2)} & \multicolumn{1}{l|}{Cayley Filter \cite{cayley}} & \multicolumn{1}{c}{58.2} & \multicolumn{1}{c}{75.2} & \multicolumn{1}{c|}{49.4} & \multicolumn{1}{c}{83.6} & \multicolumn{1}{c}{95.3} & \multicolumn{1}{c}{80.7} \\

\multicolumn{1}{c}{(3)} & \multicolumn{1}{l|}{Chebyshev Filter \cite{zhu2021unifying}} & \multicolumn{1}{c}{57.9} & \multicolumn{1}{c}{74.3} & \multicolumn{1}{c|}{50.0} & \multicolumn{1}{c}{83.7} & \multicolumn{1}{c}{94.5} & \multicolumn{1}{c}{78.2} \\

\multicolumn{1}{c}{(4)} & \multicolumn{1}{l|}{ARMA Filter \cite{bianchi2021graph}} & \multicolumn{1}{c}{58.5} & \multicolumn{1}{c}{74.5} & \multicolumn{1}{c|}{50.8} & \multicolumn{1}{c}{85.1} & \multicolumn{1}{c}{95.1} & \multicolumn{1}{c}{80.0}\\

\multicolumn{1}{c}{(5)} & \multicolumn{1}{l|}{\textbf{Band-adaptive Filter(Ours)}} & \multicolumn{1}{c}{59.2} & \multicolumn{1}{c}{74.0} & \multicolumn{1}{c|}{52.1} & \multicolumn{1}{c}{86.5} & \multicolumn{1}{c}{94.5} & \multicolumn{1}{c}{82.6} \\
\bottomrule
\end{tabular}}
\vspace{-2ex}
\end{table}

\subsubsection{Components Analysis}
The proposed SIEFormer architecture comprises two key components, the Maneuverable Filtering Layer (MFL) and the Band-Adaptive Filter (BaF), which play pivotal roles in providing distinct spectral perspectives to the Transformer model. As shown in Table \ref{comp}, it can be seen that the combined use of the BaF and MFL yields the most significant performance gains. Specifically, employing only the BaF for fine-tuning the backbone network results in notable model performance improvements. This configuration achieves performance boosts of 5.4\%, 2.1\%, and 7.1\% over SimGCD \cite{wen2023simgcd} on Stanford-Cars for All, New, and Old classes, respectively. Additionally, it achieves enhancements of 3.5\%, 1.4\%, and 4.7\% on ImageNet-100, respectively. In addition, we observe a performance dip for All and Old classes on both datasets when introducing MFL as a standalone component.

\subsubsection{Importance of Spectral Domain Parameters}
To ascertain that the use of the proposed spectral domain parameters is the key attribute to the performance improvements, we conducted experiments wherein each filter parameter was equipped with a tunable bias. This design deviates from the spectral-domain Transformer filter fine-tuning, allowing us to assess the significance of spectral parameters. Specifically, the Band-Adaptive Filter function in Eq. \eqref{BaF} was adjusted as follows:
\begin{equation}
\label{BF}
\mathcal{S}_{BaF}(\mathbf{T}, \mathbf{V})= \mathbf{T}^2\mathbf{V}\mathbf{W}_P + (\mathbf{T}^2-I_N)\mathbf{V}\mathbf{W}_R + \mathbf{W}_B,
\end{equation}
where $\mathbf{W}_B$ represents the bias term of the filter. The results of these experiments are presented in Table \ref{bias}. Our findings indicate a notable decline in model performance when additional bias terms are introduced. In all datasets, the mean reductions for All classes, Old classes, and New classes were 1.6\%, 0.2\%, and 3.2\%, respectively. These results emphasize that the simple introduction of additional parameters for fine-tuning the DINO pre-trained Vision Transformer (ViT) as the backbone network does not yield favorable outcomes. In the context of a filter applied to the \textit{values} in the self-attention mechanism, the bias term does not inherently encode information about the \textit{values} and therefore can be considered as a perturbation of the filter, which undoubtedly leads to a loss of performance.

\subsubsection{Effects of Different Filters}
Recalling that the complete form of the $1^{st}$-order approximation formed by using the Chebyshev polynomial, as a truncated expansion, effectively reduces the heavy computational load involved in using the high-order polynomials and the overfitting problem \cite{isufi2016autoregressive, kipf2016semi}, however, the use of the $1^{st}$-order Chebyshev filters suffers from noticeable deficiencies. Concretely, the Chebyshev polynomial filters are generally less sensitive to sharp changes appearing on frequencies due to the natural smoothness of polynomials, and as the network structure deepens, it is very likely that the initial node features will be weakened and even lost by continuously performing Laplacian smoothing over the graph \cite{li2018deeper, bianchi2021graph}. In addition, the eigenvalue using Chebyshev filters needs to be mapped to the interval [-1, 1], resulting in less focus on coefficients representing those small frequencies \cite{levie2018cayleynets, zhu2021unifying}.

To assess the effectiveness of the proposed approach, we replaced the Band-adaptive Filter (BaF) with various types of filters. As shown in Table \ref{Filter}, the listed filters include the $1^{st}$-order Chebyshev filter, the $1^{st}$-order ARMA filter \cite{bianchi2021graph}, all-pass filter (i.e. only the unitary matrix is used rather than the Laplacian matrix and involving no additional parameters updates), and the $1^{st}$-order Cayley filter \cite{cayley}. The $1^{st}$-order ARMA filter is a filter that is designed in a rational way and is endowed with the ability to approximate the desired filter response (i.e., $g_{\theta}(\boldsymbol{\Lambda})$) by capturing various frequency responses, accounting for higher-order neighbourhoods and being localised in the node space \cite{isufi2016autoregressive, narang2013signal, tremblay2018design, bianchi2021graph}. Specifically, based on the frequency response defined for the $K^{th}$-order ARMA filter \cite{bianchi2021graph}, $g_{\theta}^{k}\left(\lambda\right)=\frac{\sum_{k=0}^{K-1}p_{k}\lambda^{k}}{1+\sum_{k=1}^{K}q_{k}\lambda^{k}}$, the ARMA filtering operation within the node space can be derived and denoted as
\begin{equation}\label{eq13}
        \hat{\mathbf{V}}=\left(\sum_{k=0}^{K-1}p_{k}\tilde{\mathbf{L}}^{k}\right)\left(\mathbf{I}+\sum_{k=1}^{K}q_{k}\tilde{\mathbf{L}}^{k}\right)^{-1} \mathbf{V},
\end{equation}
where it reverts to the polynomial filtering by simply setting $q_{k}$ to 0 ( the MA term of the operation). By incorporating an additional AR term \cite{zhou2003learning, zhu2003semi}, the ARMA filtering operation can have the advantages of these two types of filters and exhibit its benefits in adaptively modifying the frequency responses of graph signals. Nevertheless, the matrix inversion performed in the AR term of Eq. (\ref{eq13}) involves heavy computational resources, which hinders the implementation of ARMA filters. As suggested by \cite{bianchi2021graph}, the effect of the ARMA\textsuperscript{1} filter can be reached by iterating the $1^{st}$-order recursion until convergence, as follows
\begin{equation}\label{eq14}
    \hat{\mathbf{V}}^{(t+1)}=\alpha\mathbf{T}\hat{\mathbf{V}}^{(t)}+\beta\mathbf{V},
\end{equation}
where $\mathbf{T}=\frac{1}{2}\left(\lambda_{max}-\lambda_{min}\right)\mathbf{I}-\mathbf{L}$ with $\lambda_{max}$ and $\lambda_{min}$ represent the largest and smallest eigenvalues of $\mathbf{L}$, $\alpha$ and $\beta$ are the two filter coefficients that manipulate how the ARMA filter is used to adaptively approximate the desired filter response function. For instance, the frequency response of an ARMA\textsuperscript{1} filter can be expressed as
\begin{equation}\label{eq15}
    g_{1}(\boldsymbol{\Lambda})=\frac{\beta}{1-\alpha + \alpha\boldsymbol{\Lambda}},
\end{equation}
where $\mathbf{\Lambda}$ represents all eigenvalues of $\mathbf{L}$. Note that the above ARMA\textsuperscript{1} frequency response function is also equivalent to the unnormalized form of the Personalized PageRank \cite{klicpera2018predict} if $\alpha+\beta=1$. The analytical form of Eq. (\ref{eq13}) can also be resumed by aggregating the generalized $K^{th}$-order frequency responses of the above ARMA filters and is formally expressed as
\begin{equation}\label{eq16}
    \hat{\mathbf{V}}=\mathbf{U}\sum_{k=1}^{K}g_{k}\left(\boldsymbol{\Lambda}\right) \mathbf{U}^{\operatorname{T}} \mathbf{V}.
\end{equation}

\begin{figure*}[htbp]
	\centering
    \subfigure[CUB]{\includegraphics[width=.31\textwidth]{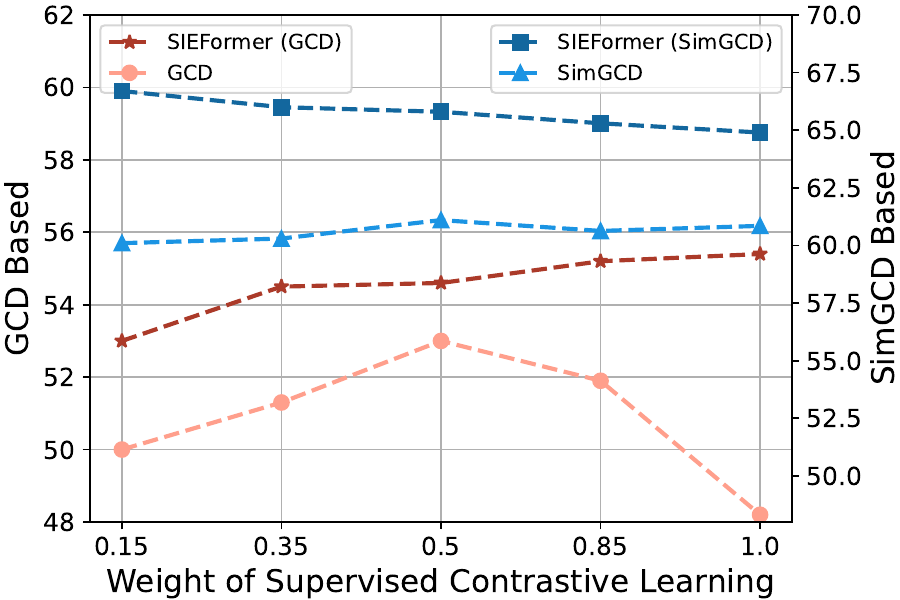}}\quad
    \subfigure[CIFAR-100]{\includegraphics[width=.31\textwidth]{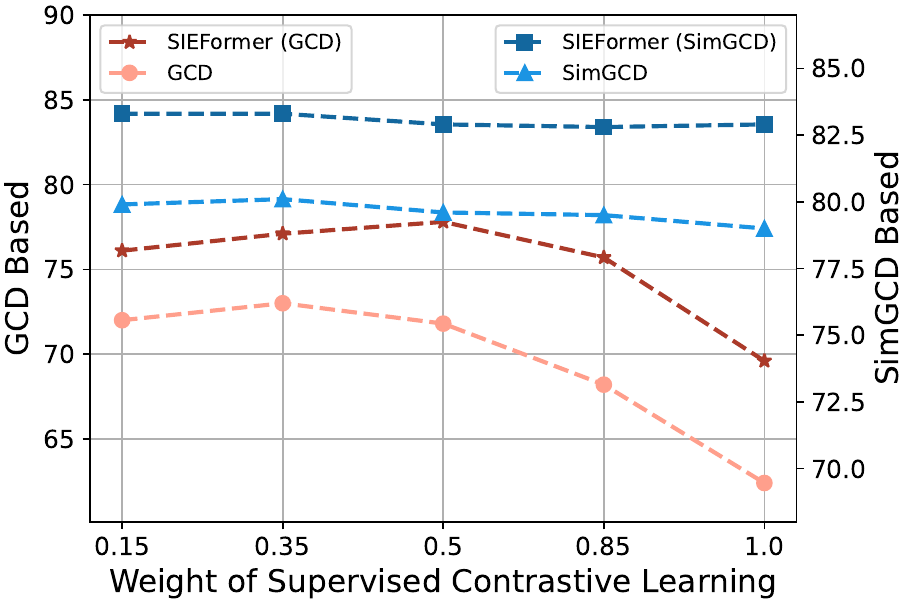}}\quad
    \subfigure[Stanford-Cars]{\includegraphics[width=.31\textwidth]{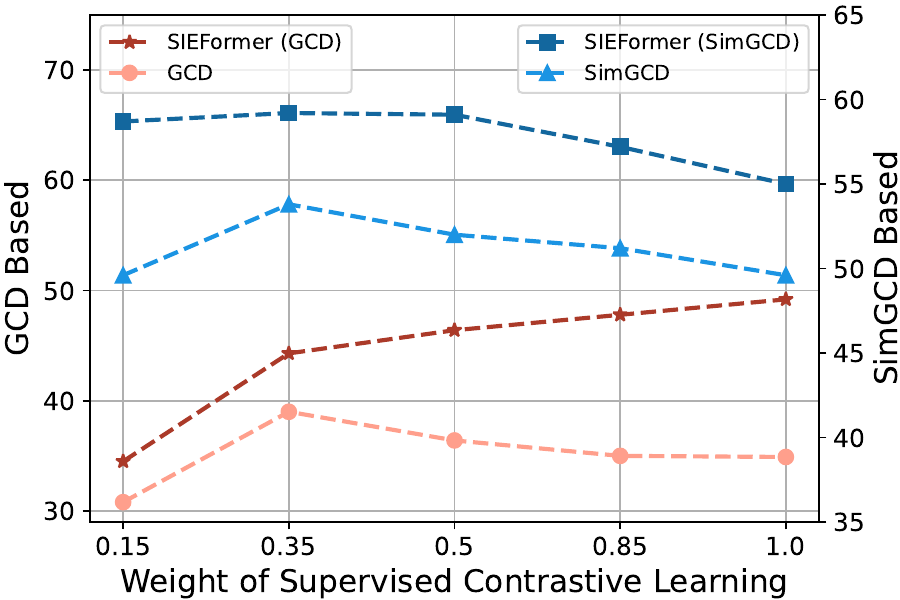}}
    \vspace{-1ex}
	\caption{Performance (\%) with different weight of supervised contrastive learning.}
	\label{fig_weigh_con}\vspace{-2ex} 
\end{figure*}

\begin{figure}[!t]
	\centering
    \subfigure{\includegraphics[width=.48\textwidth]{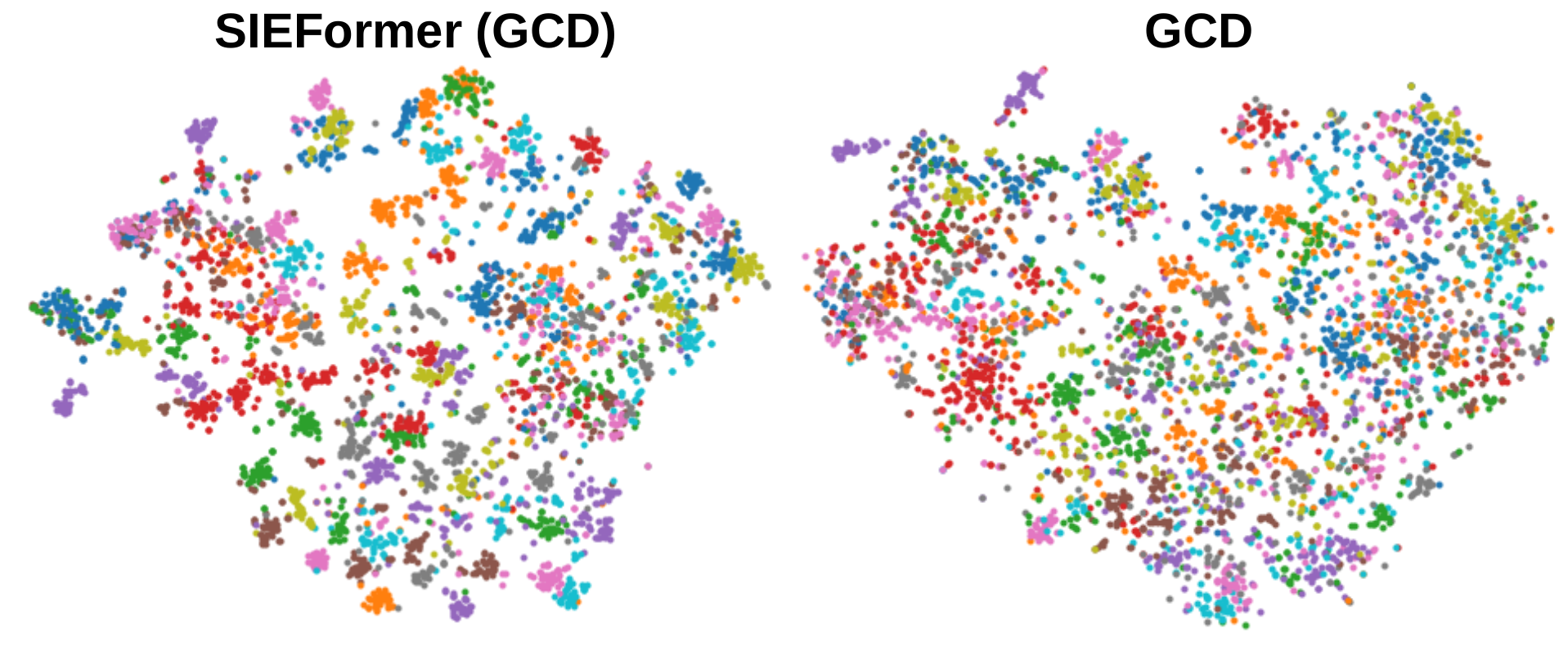}}
    \vspace{-2ex}
	\caption{The t-SNE visualization from new categories on Stanford-Cars with only supervised contrastive learning.}
	\label{fig_tsne1}
 \vspace{-2ex}
\end{figure}

The optimal parameters of $\alpha$ and $\beta$ in the filter response function Eq. \eqref{eq15} can be acquired by the linear regression but usually need prior knowledge \cite{isufi2016autoregressive}. When these two parameters are generalized to a multi-channel form as needed and rewritten as $\mathbf{W}_{\alpha}$ and $\mathbf{W}_ {\beta}$ (i.e., the recursion of Eq. (\ref{eq14}) will be changed to $\mathbf{V}^{(t+1)}=\mathbf{T}\mathbf{V}^{(t)}\mathbf{W}_{\alpha}+\mathbf{V}\mathbf{W}_{\beta}$ accordingly), deep learning can be leveraged to automatically optimize them in an end-to-end manner, similar to the mechanism of how to optimize a task-specific objective function \cite{bianchi2021graph}. Since the importance of the frequency response decreases as its order increases, the presented ARMA filter-based transformer is also generalized similarly to the complete form of the $1^{st}$-order Chebyshev approximation in Eq. (\ref{1cheb}) and expressed as  
\begin{equation}\label{eq17}
    \begin{split}
            \mathcal{S}_{\operatorname{ARMA}}(\mathbf{T}, \mathbf{V}) =\mathbf{V}\mathbf{W}_{\tau} + \left(\mathbf{T}\mathbf{V}\mathbf{W}_{\tau}\mathbf{W}_{\alpha}+\mathbf{V}\mathbf{W}_{\beta}\right),\\ 
    \end{split}
\end{equation}
where $\mathbf{W}_{\alpha}$, $\mathbf{W}_{\beta}$ and $\mathbf{W}_{\tau}$ are parameter matrices of shape $\mathbb{R}^{C \times C}$. 

We also evaluated the performance of using the Cayley filters \cite{cayley} which are known for their ability to produce smooth filtering operations while being localized in the node space. Concretely, CayleyNet \cite{cayley} introduces an additional spectral zoom parameter $h$, which is inclusive to Cayley filter \cite{zhu2021unifying} select most interested frequency bands. Formally, the $1^{st}$-order Cayley filter is defined as
\begin{equation}
    \label{cayley}
    \mathcal{S}_{\operatorname{Cayley}}(\mathbf{T}, \mathbf{V})= \mathbf{VW}_1+2(\mathbf{L}^2\mathbf{V}\mathbf{W}_{h^2}+\mathbf{V})\mathbf{W}_2.
\end{equation}

The comparative results of these filters are shown in Table \ref{Filter}. It is obvious that methods resembling low-pass filters, such as Chebyshev and ARMA filters, are relatively less effective when compared to our Band-adaptive Filter, particularly on New classes and All classes. Notably, the All-Pass filter does not introduce any additional parameters to bias, meaning that the backbone network output in the initialization phase, yields suboptimal performance and is not well-suited to the current scenario.

\begin{table}[!t]
\centering
\caption{Performance of SIEFormer, SimGCD, $\mu$GCD and SPTNet with DINOv2 on CUB, Stanford-Cars and ImageNet-100. `-’ means not reported.}
\label{dinov2}
\vspace{-2ex}
\resizebox{0.49\textwidth}{!}{
\begin{tabular}
{ccccccccccc}
\toprule
\multicolumn{1}{c}{} & \multicolumn{1}{c}{} & \multicolumn{3}{c}{CUB-200} & \multicolumn{3}{c}{Stanford-Cars} & \multicolumn{3}{c}{ImageNet-100} \\ \cmidrule(l){3-11}

\multicolumn{1}{l}{Method} & \multicolumn{1}{c}{Pre-training} & \multicolumn{1}{c}{All} & \multicolumn{1}{c}{Old} & \multicolumn{1}{c|}{New} & \multicolumn{1}{c}{All} & \multicolumn{1}{c}{Old} & \multicolumn{1}{c|}{New} & \multicolumn{1}{c}{All} & \multicolumn{1}{c}{Old} & \multicolumn{1}{c}{New}  \\ \cmidrule(l){1-11}

\multicolumn{1}{l}{GCD} & \multicolumn{1}{c|}{DINO} & \multicolumn{1}{c}{53.1} & \multicolumn{1}{c}{46.6} & \multicolumn{1}{c|}{48.7} & \multicolumn{1}{c}{39.0} & \multicolumn{1}{c}{57.6} & \multicolumn{1}{c|}{29.9} & \multicolumn{1}{c}{74.1} & \multicolumn{1}{c}{89.8} & \multicolumn{1}{c}{66.3}\\

\multicolumn{1}{l}{SimGCD} & \multicolumn{1}{c|}{DINO} & \multicolumn{1}{c}{60.3} & \multicolumn{1}{c}{65.6} & \multicolumn{1}{c|}{57.7} & \multicolumn{1}{c}{53.8} & \multicolumn{1}{c}{71.9} & \multicolumn{1}{c|}{45.0}  & \multicolumn{1}{c}{83.0} & \multicolumn{1}{c}{93.1} & \multicolumn{1}{c}{77.9}\\

\multicolumn{1}{l}{$\mu$GCD} & \multicolumn{1}{c|}{DINO} & \multicolumn{1}{c}{65.7} & \multicolumn{1}{c}{68.0} & \multicolumn{1}{c|}{64.6} & \multicolumn{1}{c}{56.5} & \multicolumn{1}{c}{68.1} & \multicolumn{1}{c|}{50.9} & \multicolumn{1}{c}{-} & \multicolumn{1}{c}{-} & \multicolumn{1}{c}{-}\\

\multicolumn{1}{l}{SPTNet} & \multicolumn{1}{c|}{DINO} & \multicolumn{1}{c}{65.8} & \multicolumn{1}{c}{68.8} & \multicolumn{1}{c|}{65.1} & \multicolumn{1}{c}{59.0} & \multicolumn{1}{c}{79.2} & \multicolumn{1}{c|}{49.3} & \multicolumn{1}{c}{85.5} & \multicolumn{1}{c}{93.9} & \multicolumn{1}{c}{81.2} \\

\multicolumn{1}{l}{Ours} & \multicolumn{1}{c|}{DINO} & \multicolumn{1}{c}{66.0} & \multicolumn{1}{c}{71.6} & \multicolumn{1}{c|}{63.2} & \multicolumn{1}{c}{59.2} & \multicolumn{1}{c}{74.0} & \multicolumn{1}{c|}{52.1} & \multicolumn{1}{c}{86.5} & \multicolumn{1}{c}{94.5} & \multicolumn{1}{c}{82.6} \\ \midrule

\multicolumn{1}{l}{GCD} & \multicolumn{1}{c|}{DINOv2} & \multicolumn{1}{c}{71.9} & \multicolumn{1}{c}{71.2} & \multicolumn{1}{c|}{72.3} & \multicolumn{1}{c}{65.7} & \multicolumn{1}{c}{67.8} & \multicolumn{1}{c|}{64.7} & \multicolumn{1}{c}{-} & \multicolumn{1}{c}{-} & \multicolumn{1}{c}{-} \\

\multicolumn{1}{l}{SimGCD} & \multicolumn{1}{c|}{DINOv2} & \multicolumn{1}{c}{71.5} & \multicolumn{1}{c}{78.1} & \multicolumn{1}{c|}{68.3} & \multicolumn{1}{c}{71.5} & \multicolumn{1}{c}{81.9} & \multicolumn{1}{c|}{66.6} & \multicolumn{1}{c}{88.5} & \multicolumn{1}{c}{86.2} & \multicolumn{1}{c}{84.6} \\

\multicolumn{1}{l}{$\mu$GCD} & \multicolumn{1}{c|}{DINOv2} & \multicolumn{1}{c}{74.0} & \multicolumn{1}{c}{75.9} & \multicolumn{1}{c|}{\textbf{73.1}} & \multicolumn{1}{c}{76.1} & \multicolumn{1}{c}{\textbf{91.0}} & \multicolumn{1}{c|}{68.9} & \multicolumn{1}{c}{-} & \multicolumn{1}{c}{-} & \multicolumn{1}{c}{-} \\

\multicolumn{1}{l}{SPTNet} & \multicolumn{1}{c|}{DINOv2} & \multicolumn{1}{c}{76.3} & \multicolumn{1}{c}{79.5} & \multicolumn{1}{c|}{74.6} & \multicolumn{1}{c}{-} & \multicolumn{1}{c}{-} & \multicolumn{1}{c|}{-} & \multicolumn{1}{c}{90.1} & \multicolumn{1}{c}{\textbf{96.1}} & \multicolumn{1}{c}{87.1} \\

\multicolumn{1}{l}{Ours} & \multicolumn{1}{c|}{DINOv2} & \multicolumn{1}{c}{\textbf{77.9}} & \multicolumn{1}{c}{\textbf{79.7}} & \multicolumn{1}{c|}{77.0} & \multicolumn{1}{c}{\textbf{77.2}} & \multicolumn{1}{c}{85.8} & \multicolumn{1}{c|}{\textbf{73.1}} & \multicolumn{1}{c}{\textbf{90.3}} & \multicolumn{1}{c}{96.0} & \multicolumn{1}{c}{\textbf{87.5}} \\
\bottomrule
\end{tabular}}
\vspace{-5ex}
\end{table}

\subsubsection{Effects of Supervised Contrastive Learning.}
We conducted experiments on three datasets, including CIFAR-100, CUB, and Stanford-Cars, to evaluate the impact of varying weights between supervised contrastive learning and self-supervised contrastive learning on overall performance. As illustrated in Fig. \ref{fig_weigh_con}, increasing the weight enhances the performance of GCD-based SIEFormer on fine-grained datasets, while the coarse-grained dataset CIFAR-100 achieves the best performance with a weight of 0.5. Due to the presence of the mean-entropy maximization regularizer and knowledge distillation in SimGCD, which effectively leverage information from unlabeled data, it is less sensitive to the weighting between supervised and self-supervised contrastive learning. Notably, when using only supervised contrastive learning, SIEFormer (GCD) attains the highest performance on fine-grained datasets. This suggests that the model can effectively recognize new categories even without any information about new classes. We also performed t-SNE visualizations of samples from new categories with the supervised contrastive learning weight \cite{GCD} set to 1 (see Fig. \ref{fig_tsne1}). The results demonstrate that our method excels in extracting both general and class-specific features.

\begin{figure*}[!t]
\centering
\includegraphics[width=0.98\textwidth]{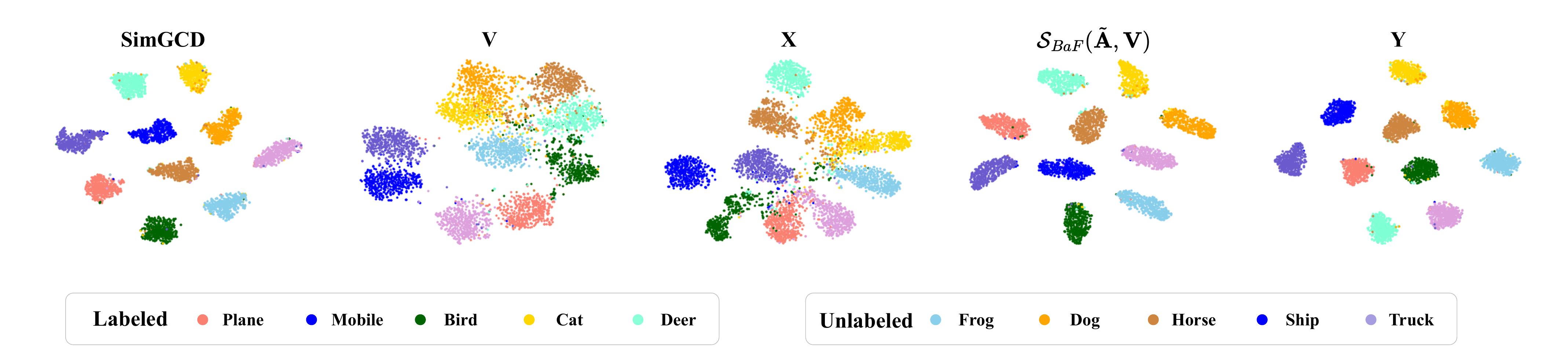}%
\vspace{-2ex}
\caption{The t-SNE visualization on CIFAR-10, includes Embeddings from SimGCD, \textit{values} $\mathbf{V}$ from the attention mechanism of last block in SIEFormer tuned ViT backbone network, output of attention mechanism from last block $\mathbf{X}$, output of the filter of last block $\mathcal{S}_{BaF}(\mathbf{T}, \mathbf{V})$, and fused features $\mathbf{Y}$ from last block.}
\label{tsne}
\vspace{-4ex}
\end{figure*}

\subsubsection{Effects of Different Pre-training Methods}
The recently proposed DINOv2 has demonstrated capabilities surpassing DINO. In order to harness the full potential of SIEFormer, we also explore the impact of fine-tuning SIEFormer with different backbone network pre-training parameters. Specifically, we investigate the substitution of the DINO \cite{DINO} pre-trained ViT-B/16 with the DINOv2 \cite{oquab2024dinov} pre-trained ViT-B/14. We conducted experiments on three datasets, CUB-200, Stanford-Cars and Imagenet-100, with the results presented in Table \ref{dinov2}. It can be found that the integration of higher-performing initialization parameters leads to significant performance improvements in SIEFormer. On Stanford-Cars, this improvement is particularly striking, with an average improvement of 15.8\% observed for both New classes and Old classes. While the enhancement on the ImageNet-100 is more modest, it nonetheless underscores the importance of pre-training parameters in shaping model performance. Since training parameters act as a decisive role in category discovery, and considering SIEFormer's heightened reliance on pre-training parameters during fine-tuning, the model's performance is directly tied to the quality of the employed pre-training parameters.

\subsubsection{Unknown Category Number}
In this paper, we directly use the actual number of categories (GT) for classification. Since the actual number of classes often differs from the estimated number of classes (Est.), we used an off-the-shelf method proposed by GCD to evaluate performance on fine-grained data, CUB-200 and Stanford-Cars, and the generic dataset ImageNet-100. As shown in Table \ref{est}, using an estimated number of classes leads to a slight performance improvement on Stanford-Cars, while performance deteriorates on CUB-200 and ImageNet-100.
However, using the estimated number, SimGCD with SIEFormer has achieved leading results across multiple dimensions.

\begin{table}[!t]
\centering
\caption{Performance and efficiency of different methods on CUB-200.}
\vspace{-2ex}
\label{tab_Efficiency}
\resizebox{0.48\textwidth}{!}{
\begin{tabular}
{ccccccc}
\toprule
\multicolumn{1}{c}{} & \multirow{2}{*}{\shortstack{Memory\\(GiB)}} & \multirow{2}{*}{Training time} & \multirow{2}{*}{Inference time} & \multicolumn{3}{c}{CUB-200} \\ \cmidrule(l){5-7}

\multicolumn{1}{l}{Methods} & & & & \multicolumn{1}{c}{All} & \multicolumn{1}{c}{Old} & \multicolumn{1}{c}{New} \\ \cmidrule(l){1-7}

\multicolumn{1}{l|}{GCD (ViT)} & \multicolumn{1}{c}{5.75} & \multicolumn{1}{c}{2.9h} & 9m & \multicolumn{1}{c}{51.3} & \multicolumn{1}{c}{56.6} & \multicolumn{1}{c}{48.7} \\

\multicolumn{1}{l|}{SimGCD (ViT)} & \multicolumn{1}{c}{5.90} & \multicolumn{1}{c}{3.1h} & 18s & \multicolumn{1}{c}{60.3} & \multicolumn{1}{c}{65.6} & \multicolumn{1}{c}{57.7}\\ 

\multicolumn{1}{l|}{SPTNet (ViT)} & \multicolumn{1}{c}{21.23} & \multicolumn{1}{c}{28.8h} & 18s & \multicolumn{1}{c}{65.8} & \multicolumn{1}{c}{68.8} & \multicolumn{1}{c}{65.1}\\ \midrule

\multicolumn{1}{l|}{\textbf{SimGCD (SIEFormer)}} & \multicolumn{1}{c}{8.09} & \multicolumn{1}{c}{3.5h} & 11m & \multicolumn{1}{c}{66.0} & \multicolumn{1}{c}{71.6} & \multicolumn{1}{c}{63.2}\\

\multicolumn{1}{l|}{\textbf{SPTNet (SIEFormer)}} & \multicolumn{1}{c}{23.80} & \multicolumn{1}{c}{31.5h} & 21s & \multicolumn{1}{c}{\textbf{69.1}} & \multicolumn{1}{c}{\textbf{75.1}} & \multicolumn{1}{c}{\textbf{66.2}}\\

\bottomrule
\end{tabular}}
\vspace{-4ex}
\end{table}

\subsubsection{Performance and Efficiency}
To better demonstrate the advantages of SIEFormer, we compared the performance and efficiency of various methods on the CUB-200 dataset. All models were evaluated on a single NVIDIA GeForce RTX 3090 GPU, and the results are summarized in Table \ref{tab_Efficiency}. To ensure a fair comparison, no inference was performed during the training phase. Experimental results show that, compared to the state-of-the-art model SPTNet (ViT), SimGCD (SIEFormer) significantly reduces both computation time and memory usage while achieving superior performance. In addition, integrating SIEFormer into SPTNet yields a 3.3\% performance improvement with only a slight increase in computational overhead.

\begin{table}[!t]
\centering
\caption{Performance of SIEFormer, SimGCD and SPTNet with an estimated number of categories on CUB, Stanford-Cars and ImageNet-100. `-’ means not reported.}
\label{est}
\resizebox{0.49\textwidth}{!}{
\begin{tabular}
{ccccccccccc}
\toprule
\multicolumn{1}{c}{} & \multicolumn{1}{c}{} & \multicolumn{3}{c}{CUB-200} & \multicolumn{3}{c}{Stanford-Cars} & \multicolumn{3}{c}{ImageNet-100} \\ \cmidrule(l){3-11}

\multicolumn{1}{l}{Method} & \multicolumn{1}{c}{$\vert C\vert$} & \multicolumn{1}{c}{All} & \multicolumn{1}{c}{Old} & \multicolumn{1}{c|}{New} & \multicolumn{1}{c}{All} & \multicolumn{1}{c}{Old} & \multicolumn{1}{c|}{New} & \multicolumn{1}{c}{All} & \multicolumn{1}{c}{Old} & \multicolumn{1}{c}{New}  \\ \cmidrule(l){1-11}

\multicolumn{1}{l}{GCD} & \multicolumn{1}{c|}{GT} & \multicolumn{1}{c}{51.3} & \multicolumn{1}{c}{56.6} & \multicolumn{1}{c|}{48.7} & \multicolumn{1}{c}{39.0} & \multicolumn{1}{c}{57.6} & \multicolumn{1}{c|}{29.9} & \multicolumn{1}{c}{74.1} & \multicolumn{1}{c}{89.8} & \multicolumn{1}{c}{66.3} \\

\multicolumn{1}{l}{SimGCD} & \multicolumn{1}{c|}{GT} & \multicolumn{1}{c}{60.3} & \multicolumn{1}{c}{65.6} & \multicolumn{1}{c|}{57.7} & \multicolumn{1}{c}{53.8} & \multicolumn{1}{c}{71.9} & \multicolumn{1}{c|}{45.0}  & \multicolumn{1}{c}{83.0} & \multicolumn{1}{c}{93.1} & \multicolumn{1}{c}{77.9}\\

\multicolumn{1}{l}{$\mu$GCD} & \multicolumn{1}{c|}{GT} & \multicolumn{1}{c}{65.7} & \multicolumn{1}{c}{68.0} & \multicolumn{1}{c|}{64.6} & \multicolumn{1}{c}{56.5} & \multicolumn{1}{c}{68.1} & \multicolumn{1}{c|}{50.9}  & \multicolumn{1}{c}{-} & \multicolumn{1}{c}{-} & \multicolumn{1}{c}{-}\\

\multicolumn{1}{l}{SPTNet} & \multicolumn{1}{c|}{GT} & \multicolumn{1}{c}{65.8} & \multicolumn{1}{c}{68.8} & \multicolumn{1}{c|}{\textbf{65.1}} & \multicolumn{1}{c}{59.0} & \multicolumn{1}{c}{\textbf{79.2}} & \multicolumn{1}{c|}{49.3} & \multicolumn{1}{c}{85.5} & \multicolumn{1}{c}{93.9} & \multicolumn{1}{c}{81.2} \\

\multicolumn{1}{l}{Ours} & \multicolumn{1}{c|}{GT} & \multicolumn{1}{c}{\textbf{66.0}} & \multicolumn{1}{c}{71.6} & \multicolumn{1}{c|}{63.2} & \multicolumn{1}{c}{59.2} & \multicolumn{1}{c}{74.0} & \multicolumn{1}{c|}{\textbf{52.1}} & \multicolumn{1}{c}{\textbf{86.5}} & \multicolumn{1}{c}{94.5} & \multicolumn{1}{c}{\textbf{82.6}} \\ \midrule

\multicolumn{1}{l}{GCD} & \multicolumn{1}{c|}{w/ Est.} & \multicolumn{1}{c}{47.1} & \multicolumn{1}{c}{55.1} & \multicolumn{1}{c|}{44.8} & \multicolumn{1}{c}{35.0} & \multicolumn{1}{c}{56.0} & \multicolumn{1}{c|}{24.8} & \multicolumn{1}{c}{72.7} & \multicolumn{1}{c}{91.8} & \multicolumn{1}{c}{63.8} \\

\multicolumn{1}{l}{SimGCD} & \multicolumn{1}{c|}{w/ Est.} & \multicolumn{1}{c}{61.5} & \multicolumn{1}{c}{66.4} & \multicolumn{1}{c|}{59.1} & \multicolumn{1}{c}{49.1} & \multicolumn{1}{c}{65.1} & \multicolumn{1}{c|}{41.3} & \multicolumn{1}{c}{81.7} & \multicolumn{1}{c}{91.2} & \multicolumn{1}{c}{76.8} \\

\multicolumn{1}{l}{$\mu$GCD} & \multicolumn{1}{c|}{w/ Est.} & \multicolumn{1}{c}{62.0} & \multicolumn{1}{c}{60.3} & \multicolumn{1}{c|}{62.8} & \multicolumn{1}{c}{56.3} & \multicolumn{1}{c}{66.8} & \multicolumn{1}{c|}{51.1} & \multicolumn{1}{c}{-} & \multicolumn{1}{c}{-} & \multicolumn{1}{c}{-} \\

\multicolumn{1}{l}{SPTNet} & \multicolumn{1}{c|}{w/ Est.} & \multicolumn{1}{c}{65.2} & \multicolumn{1}{c}{71.0} & \multicolumn{1}{c|}{62.3} & \multicolumn{1}{c}{-} & \multicolumn{1}{c}{-} & \multicolumn{1}{c|}{-} & \multicolumn{1}{c}{83.4} & \multicolumn{1}{c}{91.8} & \multicolumn{1}{c}{74.6} \\

\multicolumn{1}{l}{Ours} & \multicolumn{1}{c|}{w/ Est.} & \multicolumn{1}{c}{65.6} & \multicolumn{1}{c}{\textbf{73.5}} & \multicolumn{1}{c|}{61.7} & \multicolumn{1}{c}{\textbf{60.1}} & \multicolumn{1}{c}{78.4} & \multicolumn{1}{c|}{51.2} & \multicolumn{1}{c}{83.4} & \multicolumn{1}{c}{\textbf{94.6}} & \multicolumn{1}{c}{77.7} \\
\bottomrule
\end{tabular}}
\vspace{-1ex}
\end{table}

\begin{table}[!t]
\centering
\caption{Performance of SIEFormer, AdaptFormer, VPT and CLIP-Adapter on CUB, Stanford-Cars and CIFAR-100.}
\label{tab_zeroshot}
\vspace{-2ex}
\resizebox{0.49\textwidth}{!}{
\begin{tabular}
{cccccccccc}
\toprule
\multicolumn{1}{c}{} & \multicolumn{3}{c}{CUB-200} & \multicolumn{3}{c}{Stanford-Cars} & \multicolumn{3}{c}{CIFAR-100} \\ \cmidrule(l){2-10}

\multicolumn{1}{l}{Method} & \multicolumn{1}{c}{All} & \multicolumn{1}{c}{Old} & \multicolumn{1}{c|}{New} & \multicolumn{1}{c}{All} & \multicolumn{1}{c}{Old} & \multicolumn{1}{c|}{New} & \multicolumn{1}{c}{All} & \multicolumn{1}{c}{Old} & \multicolumn{1}{c}{New}  \\ \cmidrule(l){1-10}

\multicolumn{1}{l}{Baseline} & \multicolumn{1}{c}{60.3} & \multicolumn{1}{c}{65.6} & \multicolumn{1}{c|}{57.7} & \multicolumn{1}{c}{53.8} & \multicolumn{1}{c}{71.9} & \multicolumn{1}{c|}{45.0}  & \multicolumn{1}{c}{83.0} & \multicolumn{1}{c}{93.1} & \multicolumn{1}{c}{77.9}\\

\multicolumn{1}{l}{AdaptFormer} & \multicolumn{1}{c}{52.4} & \multicolumn{1}{c}{64.3} & \multicolumn{1}{c|}{46.5} & \multicolumn{1}{c}{32.6} & \multicolumn{1}{c}{57.5} & \multicolumn{1}{c|}{46.4}  & \multicolumn{1}{c}{81.9} & \multicolumn{1}{c}{80.8} & \multicolumn{1}{c}{87.0}\\

\multicolumn{1}{l}{VPT} & \multicolumn{1}{c}{43.5} & \multicolumn{1}{c}{64.7} & \multicolumn{1}{c|}{32.9} & \multicolumn{1}{c}{30.9} & \multicolumn{1}{c}{57.7} & \multicolumn{1}{c|}{17.9} & \multicolumn{1}{c}{80.7} & \multicolumn{1}{c}{81.1} & \multicolumn{1}{c}{79.9}\\

\multicolumn{1}{l}{CLIP-Adapter} & \multicolumn{1}{c}{50.2} & \multicolumn{1}{c}{66.4} & \multicolumn{1}{c|}{44.8} & \multicolumn{1}{c}{32.5} & \multicolumn{1}{c}{51.6} & \multicolumn{1}{c|}{23.3} & \multicolumn{1}{c}{78.5} & \multicolumn{1}{c}{82.3} & \multicolumn{1}{c}{70.8} \\

\multicolumn{1}{l}{Ours} & \multicolumn{1}{c}{66.0} & \multicolumn{1}{c}{71.6} & \multicolumn{1}{c|}{63.2} & \multicolumn{1}{c}{59.2} & \multicolumn{1}{c}{74.0} & \multicolumn{1}{c|}{52.1} & \multicolumn{1}{c}{86.5} & \multicolumn{1}{c}{94.5} & \multicolumn{1}{c}{82.6} \\
\bottomrule
\end{tabular}}
\vspace{-3ex}
\end{table}

\subsubsection{Comparison with Other ViTs}
To further validate the effectiveness of SIEFormer, we compare it with different fine-tuning paradigms. Specifically, we evaluated the performance of our proposed SIEFormer, integrated into SimGCD, against several state-of-the-art methods applicable to zero-shot learning, including VPT \cite{jia2022vpt}, CLIP-Adapter \cite{gao2024clipadapter}, and AdaptFormer \cite{chen2022adaptformer}, all of which do not leverage language models for zero-shot capabilities. As shown in the Table \ref{tab_zeroshot}, SIEFormer consistently outperforms the listed baselines, demonstrating the common ViT-based improvement strategies do not effectively address the GCD problem, which aligns with the findings of PromptCal \cite{promptcal}.

\subsection{Visualizations}
We use different visualization methods to reveal how SIEFormer achieves superior classification capabilities compared to other approaches, including t-SNE visualization, attention maps, and confusion matrices.

\subsubsection{Visualization with t-SNE}
To evaluate the effect of filters on signal processing, we conducted t-SNE visualizations \cite{van2008visualizing} using the CIFAR-10 dataset, as depicted in Fig. \ref{tsne}. We compared features extracted from three sources: Embeddings from SimGCD, \textit{values} $\mathbf{V}$ from the attention mechanism in SIEFormer tuned ViT backbone network, the output of attention mechanism from $\mathbf{X}$, the output of the filter of $\mathcal{S}_{BaF}(\mathbf{T}, \mathbf{V})$, and fused features $\mathbf{Y}$. Concretely, the features generated by SIEFormer exhibit superior separation, resulting in improved clustering and reduced inner-category dispersion compared to the features trained by \cite{wen2023simgcd}. Interestingly, the two sets of features complement each other, contributing to producing a more compact cluster space. 

\begin{figure}[!t]
\centering
\includegraphics[width=0.9\columnwidth]{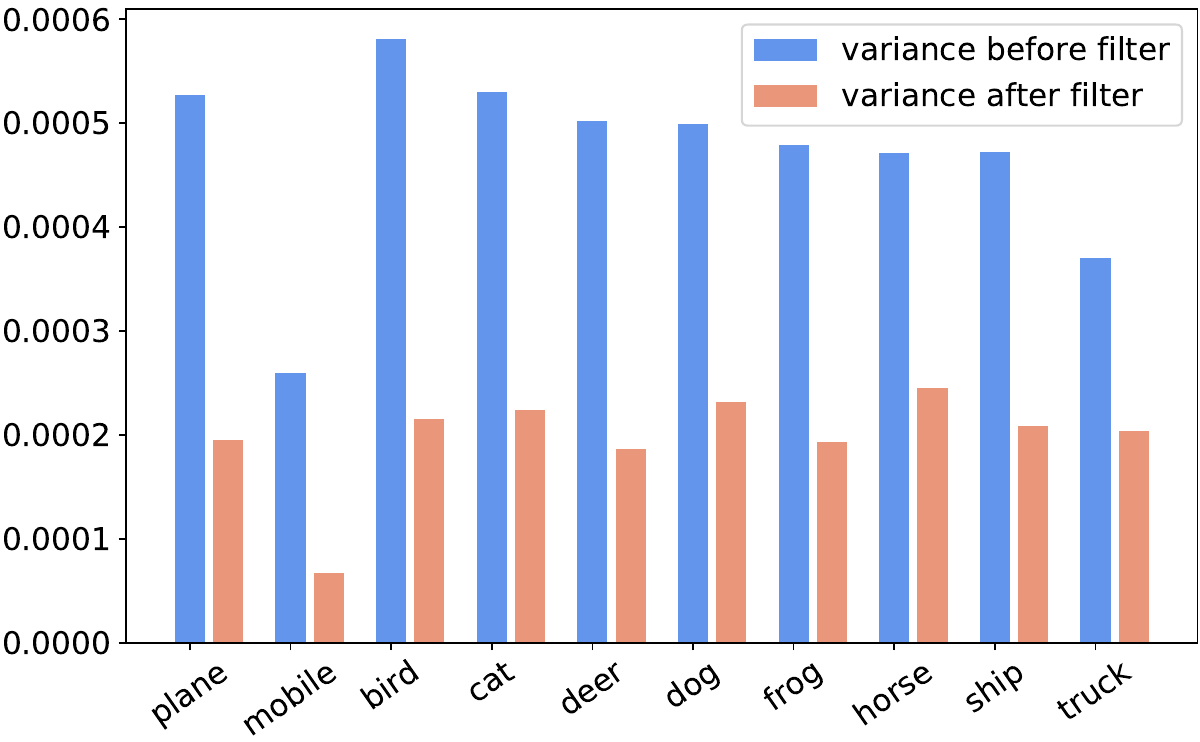}%
\vspace{-2ex}
\caption{Variances of \textit{values} before and after filter on CIFAR-10.}
\label{var}
\vspace{-2ex}
\end{figure}

To enhance the clarity of the class boundaries and compensate for the limitation of the shared axis scale, we supplemented the statistics with the measurements of feature variances from ten different classes where the features derived from the \textit{values} within the attention mechanism in the last block and the reconstructed \textit{values} generated by feeding the \textit{values} into the filters. The result is displayed in Fig. \ref{var} and statistical analysis revealed that the within-class variance was reduced by about 60\% while this reduction in within-class variance did not decrease overall variance. However, the diminished within-class variance holds significant implications for subsequent classification and clustering work as it contributes to enhanced discriminative capabilities and improved clustering performance.

\begin{figure}[!t]
\centering
\includegraphics[width=0.48\textwidth]{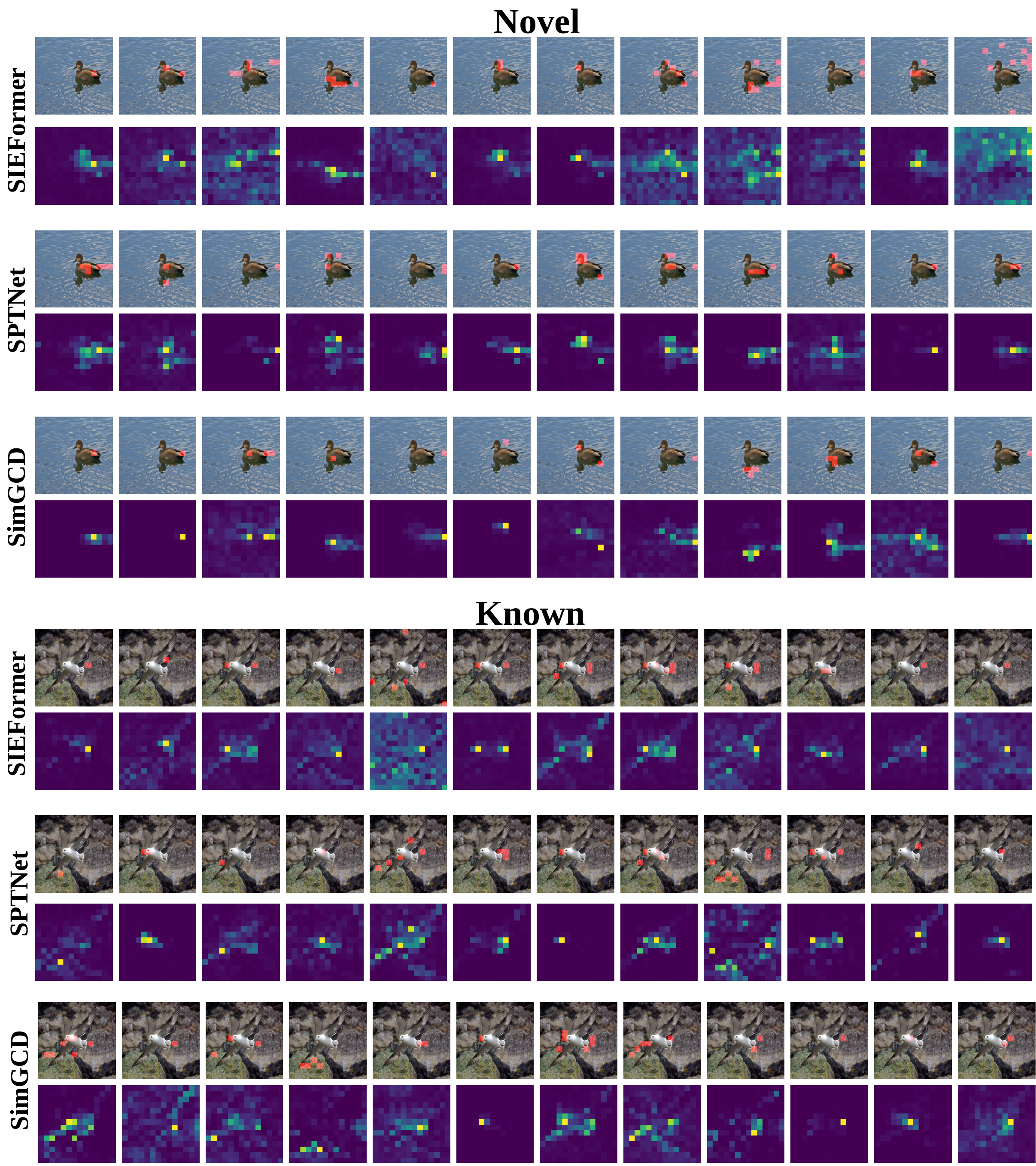}
\vspace{-2ex}
\caption{Visualization of multi-head attention on CUB-200 datasets.}
\vspace{-3ex}
\label{fig_attn1}
\end{figure}

\subsubsection{Visualization of Attention Maps}
We examined the attention maps derived from [CLS] tokens in DINO to shed light on the source of SIEFormer's superior performance. These multi-attention maps were compared with those from SimGCD \cite{wen2023simgcd} and SPTNet \cite{wang2024sptnet}. The experiments were carried out on both the Stanford-Cars and CUB-200 datasets, and the results are presented in Fig. \ref{fig_attn1} and Fig. \ref{fig_attn2}, respectively.

A comprehensive examination of the heatmaps led us to the following key observations: (1) For the CUB-200 dataset, GCD appears to prioritize holistic information about birds. In contrast, SIEFormer, leveraging the flexibility afforded by its design and the initial weights of ViT, appears to focus on certain parts of the bird (e.g., the bird's beak and feather texture) which are considered to be key to distinguishing suitable differences in fine-grained images. (2) For the Stanford-Cars dataset, SIEFormer tends to seek out model-specific features such as car brand logos and the overall structure of the vehicle, rather than broad features that other models focus on. (3) For different attention heads, some focus more on foreground regions while others focus more on background regions. In both cases, thanks to the special structure of SIEFormer, it does not lead to an excessive focus on local information; instead, it maintains a good global perspective.

\begin{figure}[!t]
\centering
\includegraphics[width=0.48\textwidth]{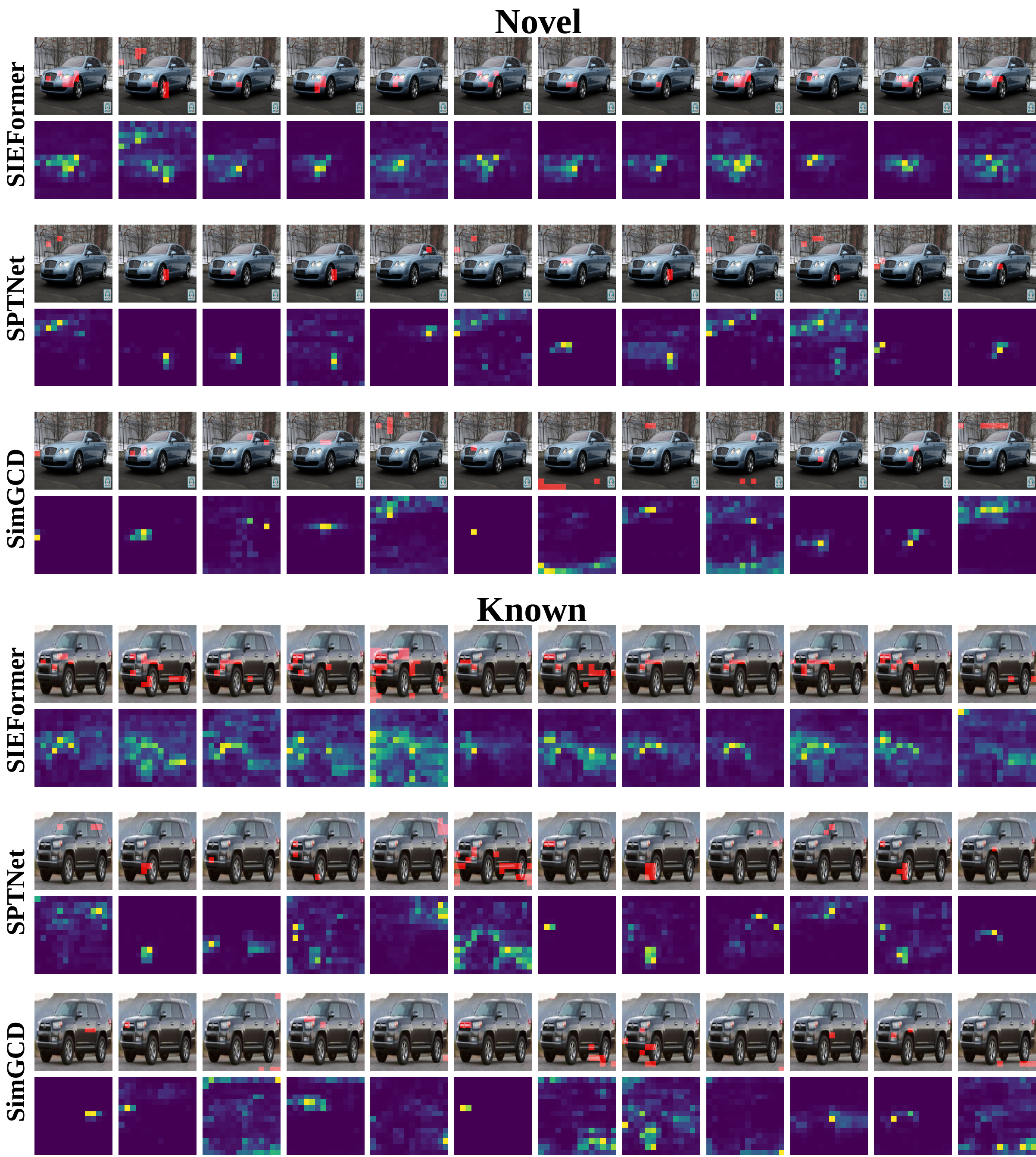}
\vspace{-2ex}
\caption{Visualization of attention on Stanford-Cars datasets.}
\label{fig_attn2}
\vspace{-3ex}
\end{figure}

\section{Conclusion and Discussion}\label{sec_5}
In this paper, we propose an innovative Vision Transformer architecture SIEFormer that is spectrally interpretable and tailored to address the unique challenges of GCD, where dedicated architectural designs are still lacking. We begin with a theoretical analysis of conventional ViT, revealing that the symmetrized and normalized affinity matrix in the attention mechanism naturally lends itself to spectral interpretation. Building on this, we introduce a dual-branch spectral ViT framework. The first branch incorporates a Band-Adaptive Filter (BaF), which emphasizes class-discriminative information by suppressing high-frequency components in the value features, with its parameters fine-tuned during training. The second branch employs a Maneuverable Filtering Layer (MFL) to extract spectral-aware features directly from the value representations, independent of the graph structure, allowing it to autonomously filter noise and mitigate the model’s bias toward known (Old) classes. By combining these two branches, SIEFormer enables a more balanced and effective optimization process. 

Our proposed method seamlessly integrates into any ViT-based architecture, achieving state-of-the-art performance across multiple public benchmark datasets with minimal compromise on model complexity and computational overhead. Future work could focus on enhancing inference efficiency, potentially through techniques such as token slimming or filter refinement. Additionally, exploring the applicability of SIEFormer to other open-world learning tasks, such as incremental learning or out-of-distribution detection, presents an exciting avenue for further research.
\section*{Acknowledgements}
This work was partly supported by the National Natural Science Foundation of China (NSFC) under Grant Nos. 62371235 and 62072246,  partly by the Key Research and Development Plan of Jiangsu Province (Industry Foresight and Key Core Technology Project) under Grant BE2023008-2, and partly by the ``111 Program" under Grant B13022.

\bibliographystyle{IEEEtran}
\bibliography{IEEE}

\begin{IEEEbiography}[{\includegraphics[width=1in,height=1.25in,clip,keepaspectratio]{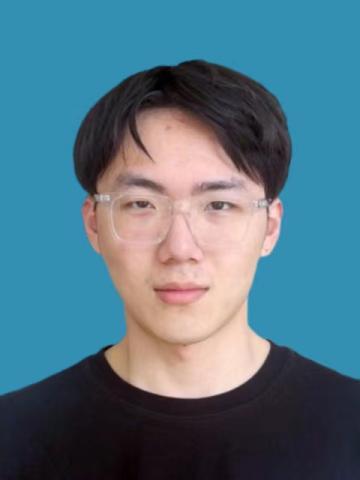}}]{Chunming Li} received the B.S. degree in Mathematics from Nanjing University of Science and Technology, Nanjing, China, in 2023, and is currently working toward the PhD degree in the School of Computer Science and Engineering, Nanjing University of Science and Technology, Nanjing, China. His current research interests include Novel Category Discovery and Few-shot Learning.
\end{IEEEbiography}

\vspace{-5ex}

\begin{IEEEbiography}[{\includegraphics[width=1in,height=1.25in,clip,keepaspectratio]{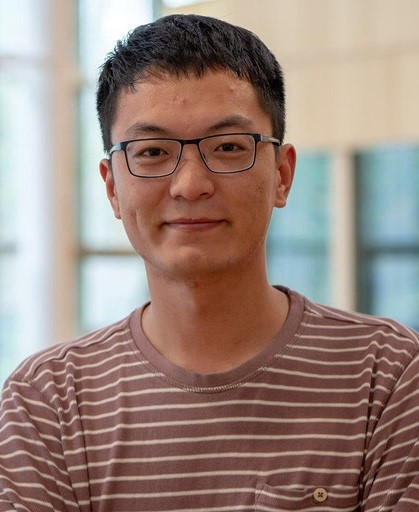}}]{Shidong Wang} is a Lecturer in Data Engineering \& AI, School of Computing, Newcastle University, UK. He received his PhD degree from the School of Computing Sciences, University of East Anglia (UEA), UK, in 2021. His research spans a breadth of domains including computer vision, deep learning, remote sensing, and environmental science, and he publishes in top-tier journals and conferences such as IEEE TPAMI, TIP, IJCV, ISPRS, AAAI and ACM MM. \end{IEEEbiography}

\vspace{-5ex}

\begin{IEEEbiography}[{\includegraphics[width=1in,height=1.25in,clip,keepaspectratio]{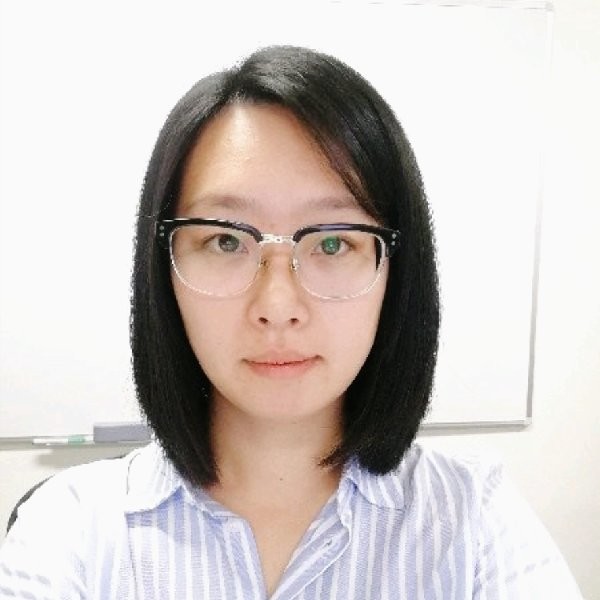}}]{Tong Xin} is a Lecturer in the School of Computing, Newcastle University. She obtained her PhD degree from the same department in 2020. Her interdisciplinary research focus on computer graphics, computer vision, medical image computing and data science. 
\end{IEEEbiography}

\vspace{-5ex}
\begin{IEEEbiography}[{\includegraphics[width=1in,height=1.25in,clip,keepaspectratio]{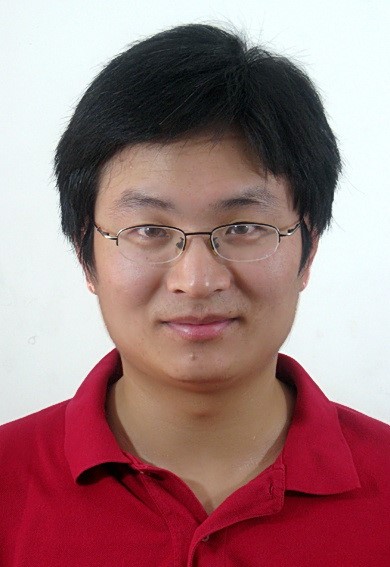}}]{Haofeng Zhang} is currently a Professor with the School of Computer Science and Engineering, Nanjing University of Science and Technology, China. He received the B.Eng. degree and the Ph.D. degree in 2003 and 2007 respectively from School of Computer Science and Technology, Nanjing University of Science and Technology, Nanjing, China. From Dec. 2016 to Dec. 2017, He was an academic visitor at University of East Anglia, Norwich, UK. His research interests include computer vision and mobile robot.
\end{IEEEbiography}

\vfill

\end{document}